\documentclass{article}
\usepackage{iclr2027_conference,times}

\usepackage{amsmath,amsfonts,bm}

\def\eqref#1{equation~\ref{#1}}
\def\1{\bm{1}}

\DeclareMathAlphabet{\mathsfit}{\encodingdefault}{\sfdefault}{m}{sl}
\SetMathAlphabet{\mathsfit}{bold}{\encodingdefault}{\sfdefault}{bx}{n}

\usepackage{hyperref}
\usepackage{url}

\usepackage[utf8]{inputenc}
\usepackage[T1]{fontenc}
\renewcommand*{\ttdefault}{cmtt}
\usepackage{booktabs}
\usepackage{amsfonts}
\usepackage{nicefrac}
\usepackage{microtype}
\usepackage{xcolor}

\usepackage{amsmath}
\usepackage{amssymb}
\usepackage{mathtools}
\usepackage{amsthm}

\usepackage{graphicx}
\usepackage{capt-of}
\usepackage{wrapfig}
\usepackage{needspace}
\usepackage{array}
\usepackage{multirow}
\usepackage{colortbl}
\usepackage{scalerel}

\usepackage{algorithm}
\usepackage[endLComment=,italicComments=false]{algpseudocodex}

\theoremstyle{plain}
\newtheorem{theorem}{Theorem}
\newtheorem{lemma}[theorem]{Lemma}
\newtheorem{proposition}[theorem]{Proposition}

\theoremstyle{definition}
\newtheorem{definition}[theorem]{Definition}
\newtheorem{assumption}[theorem]{Assumption}

\theoremstyle{remark}
\newtheorem{remark}[theorem]{Remark}

\newcommand{\Aleq}{\mathcal{A}^{\le H}}
\newcommand{\Bnh}{\mathcal{B}_{\theta}^{N,H}}

\newcommand{\Qnh}{Q_{\theta}^{N,H}}

\usepackage[capitalize]{cleveref}
\crefname{table}{Tab.}{Tabs.}
\crefname{section}{Sec.}{Secs.}
\crefname{algorithm}{Alg.}{Algs.}
\crefname{assumption}{Assumption}{Assumptions}

\definecolor{acsacblue}{HTML}{598BE7}
\definecolor{ourpurple}{HTML}{CB297B}
\definecolor{ourblue}{HTML}{0076BA}
\definecolor{ourmiddle}{HTML}{66509B}
\definecolor{ourgreen}{HTML}{00A89D}
\definecolor{ourdark}{HTML}{101010}
\definecolor{ourlightblue}{HTML}{F5FAFE}
\definecolor{ourlightmiddle}{HTML}{F7F6F9}
\definecolor{ourlightpurple}{HTML}{FDF7FA}
\definecolor{ourlightgreen}{HTML}{F2FBFA}
\definecolor{baselinecolor}{HTML}{EEEEEE}

\newcommand{\blue}[1]{\textcolor{acsacblue}{#1}}

\title{ACSAC: Adaptive Chunk Size Actor-Critic with Causal Transformer Q-Network}

\author{\textbf{Qian Chen}$^{1}$ \quad
\textbf{Junqiao Zhao}$^{1}$\thanks{Corresponding author} \quad
\textbf{Hongtu Zhou}$^{1}$ \\[2pt]
\textbf{Hang Yu}$^{1}$ \quad
\textbf{Yanping Zhao}$^{1}$ \quad
\textbf{Chen Ye}$^{1}$ \quad
\textbf{Guang Chen}$^{1}$ \\[4pt]
$^{1}$Tongji University
}

\iclrfinalcopy
\begin{document}

\maketitle

\begin{abstract}
Long-horizon, sparse-reward tasks pose a fundamental challenge for reinforcement learning,
since single-step TD learning suffers from bootstrapping error accumulation across successive Bellman updates.
Actor-critic methods with action chunking address this by operating over temporally extended actions,
which reduce the effective horizon, enable fast value backups, and support temporally consistent exploration.
However, existing methods rely on a fixed chunk size and therefore cannot adaptively balance reactivity against temporal consistency.
A large fixed chunk size reduces responsiveness to new observations,
while a small one produces incoherent motions, forcing task-specific tuning of the chunk size.
To address this limitation, we propose \textbf{A}daptive \textbf{C}hunk \textbf{S}ize \textbf{A}ctor-\textbf{C}ritic (\textbf{ACSAC}).
ACSAC leverages a causal Transformer critic to evaluate expected returns for action chunks of different sizes.
At each chunk boundary, it adaptively selects the chunk size that maximizes the expected return,
supporting flexible, state-dependent chunk sizes.
We prove that the ACSAC Bellman operator is a $\gamma$-contraction whose unique fixed point is the action-value function of the adaptive policy.
Experiments on OGBench demonstrate that ACSAC achieves strong performance on long-horizon, sparse-reward manipulation and navigation tasks across both offline RL and offline-to-online RL settings.
\end{abstract}

\section{Introduction}

Offline reinforcement learning (RL) trains policies from previously collected datasets without environment interaction~\citep{LevineSurvey},
and the offline-to-online setting further refines the offline-pretrained policy through additional online interaction~\citep{AWAC}.
However, on long-horizon, sparse-reward tasks, single-step temporal difference (TD) learning suffers from bootstrapping error accumulation,
since each update regresses toward the Q-network's own next-state estimate and small errors compound across many recursive Bellman updates~\citep{SHARSA}.
A common remedy is to use multi-step returns, which accelerate value propagation by shifting the regression target further into the future,
but they introduce an off-policy bias because the intermediate rewards are collected under the behavior policy rather than the current policy~\citep{TOP-ERL, QC}.
Action chunking has emerged as a recent and effective response to these limitations.

Action chunking trains the critic and the policy on action sequences rather than single actions,
which enables multi-step value backups without incurring off-policy bias, 
since the critic conditions on the full sequence~\citep{QC}.
Beyond these critic-side benefits, predicting full action sequences also produces temporally coherent behaviors that capture non-Markovian patterns in the data and improve exploration in sparse-reward settings~\citep{ACT, QC}.
However, existing action-chunking actor-critic methods rely on a fixed chunk size,
manually tuned per task and shared across all states~\citep{QC, DQC, DEAS}.
Yet the optimal balance between reactivity and temporal consistency is itself state-dependent.
Stable phases admit long open-loop chunks, whereas sensitive states demand frequent replanning~\citep{AAC}.
Over-committing to a long chunk in such states can drive the agent off the goal-reaching path
(\cref{fig:teaser}).

\begin{figure}[t]
\centering
\includegraphics[width=\linewidth]{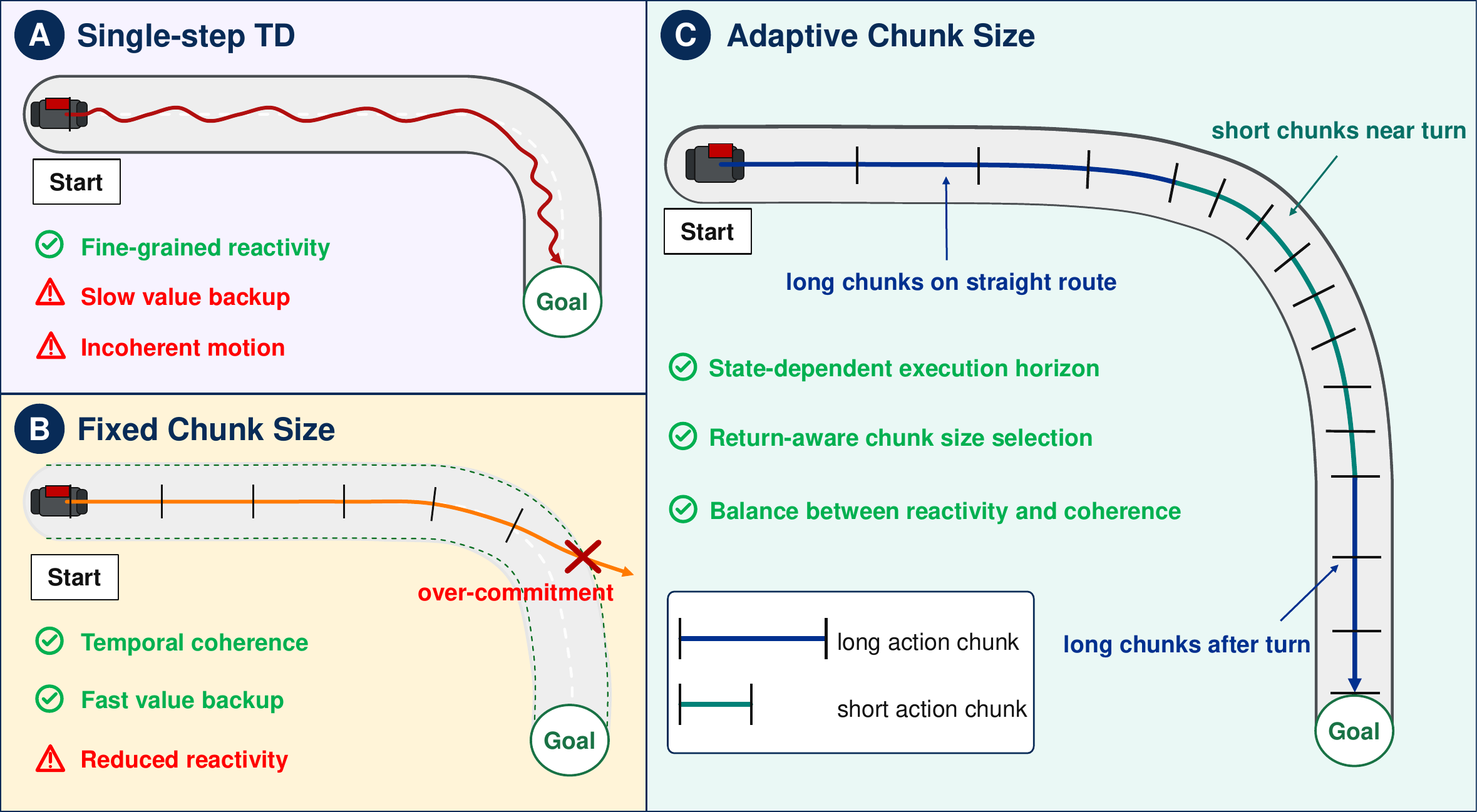}
\vspace{-14pt}
\caption{\textbf{Motivation for adaptive action chunk size.}
\textbf{(A)} Single-step execution preserves fine-grained reactivity by replanning at every step,
but suffers from slow value backups and produces incoherent motions.
\textbf{(B)} A fixed chunk size improves motion coherence and accelerates value propagation, 
but its open-loop execution reduces reactivity within the chunk.
In sensitive states such as turns, this over-commitment to a long chunk can drive the agent off the goal-reaching path, lowering the return.
\textbf{(C)} ACSAC adaptively selects the execution horizon per state via return-aware chunk size selection,
executing longer chunks on the straight route and shorter chunks near the turn,
achieving a balance between reactivity and coherence.}
\label{fig:teaser}
\vspace{-14pt}
\end{figure}

To address this limitation, we propose \textbf{A}daptive \textbf{C}hunk \textbf{S}ize \textbf{A}ctor-\textbf{C}ritic (\textbf{ACSAC}),
which adaptively selects the chunk size to maximize the critic's value estimate.
At each replanning state, ACSAC samples multiple candidate action chunks of length $H$ from an expressive flow BC policy.
A causal Transformer critic provides \textbf{calibrated and cross-horizon comparable} Q-value estimates for every \emph{prefix} (i.e., the first $h$ actions of a chunk, $h = 1, \ldots, H$) of each candidate.
ACSAC then applies rejection sampling jointly over the candidate index and the prefix length,
executes the highest-value prefix, and replans at the next chunk boundary.
ACSAC thus retains the multi-step value backups and coherent exploration of action chunking,
while the chunk size becomes state-dependent rather than fixed, 
adaptively balancing reactivity and temporal consistency at each state.

\vspace{-1pt}
Our contributions can be summarized as follows:
1) We propose \textbf{ACSAC}, an action-chunking actor-critic that adaptively selects the chunk size per state, 
using a causal Transformer critic and joint rejection sampling over candidate index and prefix length.
2) We prove that ACSAC's Bellman operator is a $\gamma$-contraction whose unique fixed point is the action-value function of the adaptive policy, and that the prefix-conditioned Q-values at this fixed point are cross-horizon comparable.
3) On long-horizon, sparse-reward manipulation and navigation tasks from OGBench~\citep{OGBench}, ACSAC achieves strong performance against single-step, multi-step, and fixed chunk size baselines in both offline and offline-to-online settings.

\vspace{-8pt}
\section{Related Work}\label{sec:related}
\vspace{-4pt}

\textbf{Offline RL and offline-to-online RL.}
Offline RL aims to learn a policy from a fixed dataset without environment interaction~\citep{LevineSurvey}.
The main challenge is the distributional shift between the behavior policy and the learned policy,
which can cause value overestimation and suboptimal performance~\citep{TD3+BC}.
Recently, expressive generative policies based on diffusion and flow matching~\citep{DDPM, FM}
have been widely adopted for their expressivity over Gaussian policies.
Common policy extraction strategies for these generative policies include reparameterized gradients~\citep{DQL, FQL},
weighted regression~\citep{QGPO, QIPO}, and rejection sampling~\citep{SfBC, IDQL}.
In the offline-to-online setting, the offline-pretrained policy is further fine-tuned with online interactions~\citep{AWAC},
with techniques such as balanced sampling~\citep{Off2On}, 
high update-to-data ratios~\citep{RLPD}, 
value calibration~\citep{CalQL}, 
and more.
Our method uses the same algorithm for both offline and online training,
simply adding online transitions to the offline dataset and applying none of the above specialized techniques.

\textbf{Action chunking.}
Action chunking originated in imitation learning, where a policy predicts and executes a sequence of actions in an open-loop manner,
improving robustness and capturing non-Markovian behavior~\citep{ACT, DP}.
Recent RL methods have brought action chunking into actor-critic frameworks,
where the critic evaluates whole action chunks and enables multi-step backup without off-policy bias.
In the online setting, CQN-AS~\citep{CQN-AS} learns a multi-level factorized critic on action chunks,
and AC3~\citep{AC3} builds on a DDPG-style framework to predict continuous action chunks.
In the offline or offline-to-online setting, Q-chunking~\citep{QC} runs RL at an action chunk level
with a flow BC policy and rejection sampling.
DQC~\citep{DQC} decouples the policy chunk size from the critic chunk size,
with the policy predicting a shorter action chunk while retaining the value learning benefits of the chunked critic.
MAC~\citep{MAC} combines an action-chunk dynamics model with rejection sampling from an expressive flow BC policy.
DEAS~\citep{DEAS} leverages action sequences for training critics with detached value learning and classification loss.
CGQ~\citep{CGQ} regularizes a single-step critic toward a chunked critic.
All of the above RL methods rely on a fixed chunk size across states.
However, the optimal chunk size may vary by state, and any fixed choice forces a single trade-off between reactivity and temporal consistency.
Our method brings adaptive, state-dependent chunk size selection to RL, driven by the critic's value estimate.

\textbf{Transformer-based Q-networks.}
Transformers have become strong backbones for the Q-network in RL.
Q-Transformer~\citep{QT} converts Q-function estimation into a discrete token sequence modeling problem, 
with per-dimension discretization treating each action dimension as a separate time step 
for autoregressive Q-value prediction.
TQL~\citep{TQL} scales Transformer Q-learning through per-layer control of attention entropy that prevents attention collapse.
While Q-Transformer and TQL apply Transformers to single-action Q-networks, 
recent work instead trains Transformer critics on action chunks for direct multi-step value evaluation.
TOP-ERL~\citep{TOP-ERL} introduces a causal Transformer critic in episodic RL,
with the policy outputting full trajectories via movement primitives.
T-SAC~\citep{T-SAC} adapts a similar causal Transformer critic to step-based RL, 
conditioning the critic on short trajectory segments while keeping the actor single-step.
SEAR~\citep{SEAR} combines a causal Transformer critic with multi-horizon targets and random replanning during data collection.
CO-RFT~\citep{CO-RFT} applies a causal Transformer critic to fine-tune vision-language-action models with offline RL.
None of these methods adapt the chunk size to state.
Our method also adopts a causal Transformer critic, 
but uniquely exploits its prefix-conditioned outputs at all prefix lengths 
to support state-dependent chunk size selection during both training and inference.

\section{Preliminaries}
\label{sec:preliminaries}

\textbf{Problem formulation.}
We consider an infinite-horizon Markov decision process (MDP) $\mathcal{M} = (\mathcal{S}, \mathcal{A}, T, r, \rho, \gamma)$, where $\mathcal{S}$ is the state space, $\mathcal{A} \subseteq \mathbb{R}^d$ is the continuous action space of dimension $d$, $T(s' | s, a): \mathcal{S} \times \mathcal{A} \to \Delta(\mathcal{S})$ is the transition dynamics distribution, $r(s, a): \mathcal{S} \times \mathcal{A} \to \mathbb{R}$ is the reward function, $\rho \in \Delta(\mathcal{S})$ is the initial state distribution, and $\gamma \in (0, 1)$ is the discount factor.
Here $\Delta(\mathcal{X})$ denotes the set of probability distributions over a space $\mathcal{X}$.
We assume access to a prior offline dataset $\mathcal{D}$ consisting of transition rollouts $\{(s, a, s', r)\}$ collected from $\mathcal{M}$.
The goal is to find a policy $\pi(a|s): \mathcal{S} \to \Delta(\mathcal{A})$ that maximizes the expected discounted return $J(\pi) := \mathbb{E}_{s_{t+1} \sim T(s_t, a_t), a_t \sim \pi(\cdot | s_t)} \left[ \sum_{t=0}^{\infty} \gamma^t r(s_t, a_t) \right]$.
In the \emph{offline} setting, the policy is learned entirely from the fixed dataset $\mathcal{D}$ without environment interactions; in the \emph{offline-to-online} setting, the offline-pretrained policy is further fine-tuned with online interactions.

\textbf{Chunk-based reinforcement learning.}
Standard TD-based methods learn a state-action value function $Q_\phi(s, a)$ by regressing toward the single-step target $r_t + \gamma Q_{\bar{\phi}}(s_{t+1}, a_{t+1})$, where $a_{t+1} \sim \pi(\cdot \mid s_{t+1})$.
Each single-step backup propagates value only one step backward, slowing learning in long-horizon tasks.
A common strategy is the multi-step return, which replaces the single-step target with $\sum_{\tau=0}^{n-1} \gamma^\tau r_{t+\tau} + \gamma^n Q_{\bar{\phi}}(s_{t+n}, a_{t+n})$ for some horizon $n \ge 1$ and allows for an $n$-fold speed-up in value propagation.
However, the multi-step return introduces off-policy bias, since the discounted reward sum from the replay buffer may not reflect the expected rewards under the current policy when the intermediate actions $a_{t+1}, \ldots, a_{t+n-1}$ are chosen by a different policy~\citep{TOP-ERL, QC}.

Action chunking attains the value-propagation speedup of multi-step returns without their off-policy bias by extending RL to action sequences.
An action chunk of length $H$ starting at time $t$ is $a_{t:t+H} := (a_t, a_{t+1}, \ldots, a_{t+H-1}) \in \mathcal{A}^H$, and the corresponding $H$-step discounted reward is $r_t^H := \sum_{\tau=0}^{H-1} \gamma^\tau r_{t+\tau}$.
The chunked critic $Q_\phi(s_t, a_{t:t+H})$ is trained with:
\begin{equation}\label{eq:chunked_td}
    \mathcal{L}^{\text{chunk}}(\phi) = \mathbb{E}_{(s_t, a_{t:t+H}, r_t^H, s_{t+H}) \sim \mathcal{D}} \left[ \left( Q_\phi(s_t, a_{t:t+H}) - r_t^H - \gamma^H Q_{\bar{\phi}}(s_{t+H}, a_{t+H:t+2H}) \right)^2 \right],
\end{equation}
where $a_{t+H:t+2H} \sim \pi(\cdot | s_{t+H})$.
Crucially, unlike the multi-step return where the single-action critic $Q(s_t, a_t)$ is backed up with rewards generated by potentially off-policy actions, the chunked critic $Q(s_t, a_{t:t+H})$ conditions on the exact action sequence used to obtain the multi-step rewards $r_t^H$, eliminating the off-policy bias~\citep{QC}.

\textbf{Flow policy.}
Flow matching~\citep{FM, ReFlow} is a generative modeling technique that trains a velocity field to transform a noise distribution into a target data distribution.
Given a target distribution $p(x) \in \Delta(\mathbb{R}^d)$, flow matching fits a time-dependent velocity field $v_\theta(u, x): [0,1] \times \mathbb{R}^d \to \mathbb{R}^d$ whose corresponding flow $\psi_\theta(u, x)$, defined by the ODE $\frac{d}{du} \psi_\theta(u, x) = v_\theta(u, \psi_\theta(u, x))$,
transforms $\mathcal{N}(0, I_d)$ at $u=0$ into $p(x)$ at $u=1$, by minimizing $\mathcal{L}(\theta) = \mathbb{E}_{x^0, x^1, u} [ \| v_\theta(u, x^u) - (x^1 - x^0) \|_2^2 ]$ with $x^0 \sim \mathcal{N}(0, I_d)$, $x^1 \sim p(x)$, and $u \sim \mathrm{Unif}([0, 1])$,
where $x^u := (1-u)x^0 + u x^1$ is the linear interpolation between $x^0$ and $x^1$.
To use flow matching for policy learning, we train a state-conditioned velocity field $v_\theta: [0,1] \times \mathcal{S} \times \mathbb{R}^d \to \mathbb{R}^d$ with the behavior cloning objective:
\begin{equation}\label{eq:flow_bc}
    \mathcal{L}(\theta) = \mathbb{E}_{\substack{s, a \sim \mathcal{D},\, z \sim \mathcal{N}(0, I_d),\, u \sim \text{Unif}([0,1])}} \left[ \| v_\theta(u, s, a^u) - (a - z) \|_2^2 \right],
\end{equation}
where $a^u := (1-u)z + ua$ is the linear interpolation between the noise $z = a^0$ and the action $a = a^1$.
We denote the ODE solution at $u=1$ starting from noise $z$ as $\pi_\theta(s, z) := \psi_\theta(1, s, z)$, which maps a noise vector $z \sim \mathcal{N}(0, I_d)$ to an action $a = \pi_\theta(s, z)$.
Compared to Gaussian policies, flow policies can model complex, multi-modal action distributions, making them particularly suited for offline RL where datasets often contain diverse behavior patterns~\citep{FQL}.

\section{Method}

ACSAC is an action-chunking actor-critic that adaptively selects the chunk size at each replanning state, rather than fixing one across all states as in prior methods.
At each replanning state $s_t$, it samples $N$ candidate chunks of length $H$ from a flow BC policy $\pi_\theta$, where $H$ is the maximum chunk size.
A causal Transformer critic $Q_\phi$ then evaluates every prefix (i.e., the first $h$ actions, $h \in [H] := \{1, \ldots, H\}$) of each candidate.
ACSAC's adaptive policy $\pi_\star$ executes the highest-value prefix, of length $h^\star$, and replans at $s_{t+h^\star}$ (\cref{fig:pipeline}).
The critic is trained at every prefix length with a multi-step TD objective, in which $\pi_\star$ also provides the bootstrap action prefix.
We describe the causal Transformer critic architecture in \cref{sec:critic_arch}, the multi-step TD objective in \cref{sec:td_objective}, and the adaptive policy extraction by joint rejection sampling in \cref{sec:adaptive_extraction}.

\begin{figure}[!t]
\centering
\includegraphics[width=1\linewidth,trim=10pt 8pt 7pt 8pt,clip]{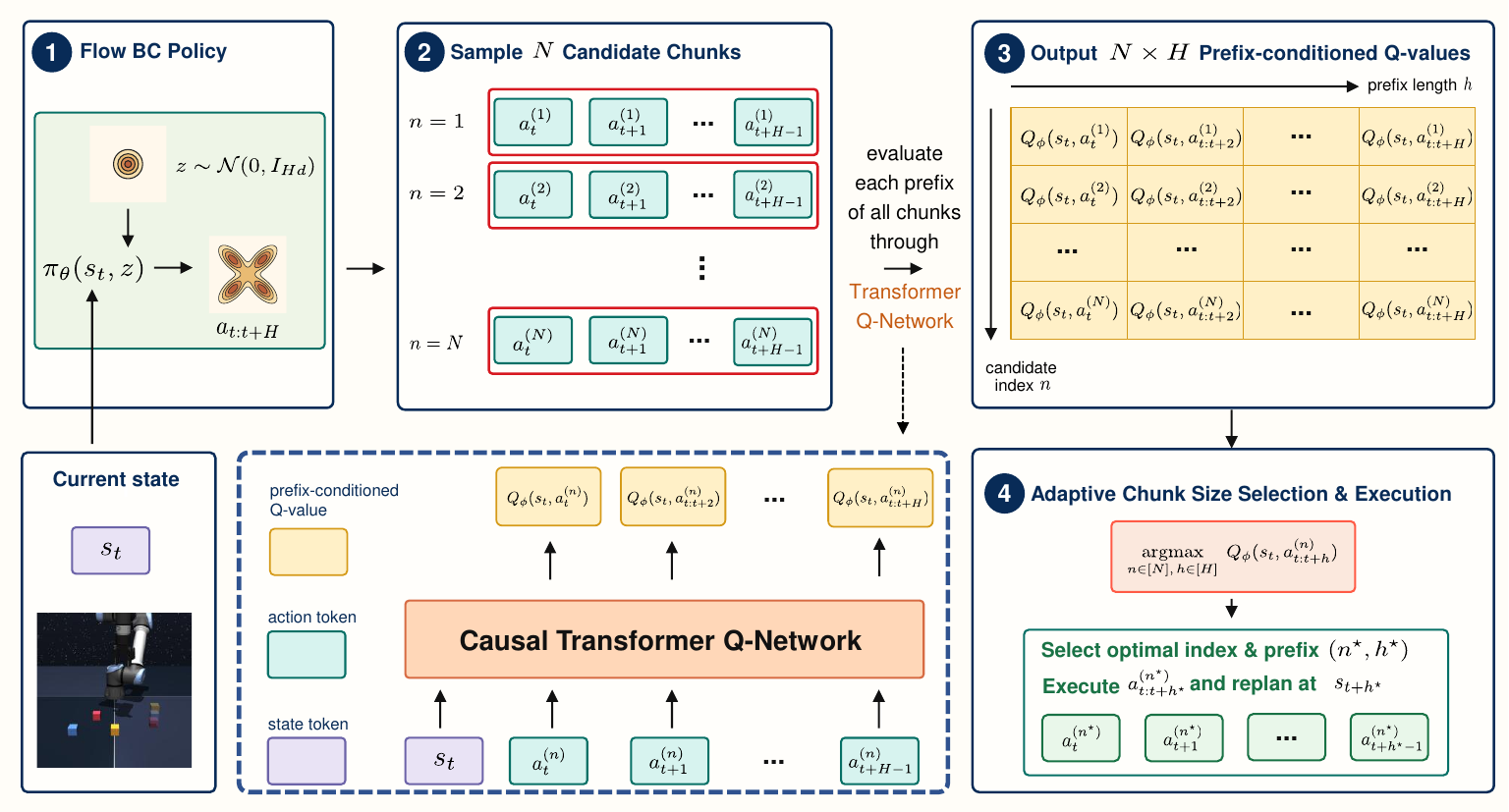}
\vspace{-10pt}
\caption{\textbf{Adaptive policy extraction in ACSAC.}
At replanning state $s_t$, the flow BC policy maps $N$ noise vectors
$z^{(n)}\sim\mathcal{N}(0,I_{Hd})$ to length-$H$ candidate chunks
$a^{(n)}_{t:t+H}=\pi_\theta(s_t,z^{(n)})$.
The causal Transformer Q-network $Q_\phi$ takes the state token and each candidate's action tokens as input and jointly outputs an $N\times H$ array of prefix-conditioned Q-values.
ACSAC selects $(n^\star,h^\star)$ by jointly maximizing $Q_\phi(s_t,a^{(n)}_{t:t+h})$ over candidate index $n\in[N]$ and prefix length $h\in[H]$.
It executes the corresponding prefix $a^{(n^\star)}_{t:t+h^\star}$ and replans at $s_{t+h^\star}$.}
\label{fig:pipeline}
\vspace{-10pt}
\end{figure}

\subsection{Causal Transformer Critic Architecture}\label{sec:critic_arch}

ACSAC's joint argmax (\cref{sec:adaptive_extraction}) requires the critic to evaluate action chunks of different lengths $a_{t:t+h}$ ($h \in [H]$) from a single state $s_t$ and to produce values that are comparable across $h$.
An MLP critic on action chunks consumes fixed-length inputs, so reusing it on shorter sub-prefixes can incur causal leakage~\citep{T-SAC}.
A causal Transformer instead ingests $(s_t, a_{t:t+H})$ as a token sequence and outputs the $H$ prefix-conditioned values $\{Q_\phi(s_t, a_{t:t+h})\}_{h \in [H]}$ jointly, with a causal attention mask that restricts position $i$ to attend only to positions $j \le i$.

This design has three consequences.
First, $Q_\phi(s_t, a_{t:t+h})$ depends only on the prefix $a_{t:t+h}$ and not on future actions $a_{t+h:t+H}$, so each value is a valid estimate for executing exactly $h$ actions and the MLP-style causal leakage is eliminated.
Second, the shared backbone produces sequence-aware value estimates that capture the temporal structure within the chunk and enable fine-grained credit assignment across positions~\citep{TOP-ERL, T-SAC}.
Third, all $H$ values share a backbone and are jointly trained against per-horizon targets (\cref{sec:td_objective}) to estimate expected total discounted return under the same adaptive policy $\pi_\star$.
As formalized in \cref{thm:cross_horizon}, these returns lie on a common return scale, making the joint argmax in $\pi_\star$ meaningful across different prefix lengths.

\begin{theorem}[Cross-horizon Q-value comparability; proof in \cref{app:theory-comparability}]
\label{thm:cross_horizon}
Let $Q^\pi(s_t, a_{t:t+h})$ denote the expected discounted return from executing $a_{t:t+h}$ and then following policy $\pi$.
For every $h \in [H]$, with $V^\pi(s)$ denoting the expected discounted return of following $\pi$ from state $s$, we have
\begin{equation}
\label{eq:Q_pi_prefix}
Q^\pi(s_t, a_{t:t+h}) \;=\; \sum_{\tau=0}^{h-1} \gamma^\tau r_{t+\tau} \;+\; \gamma^h\, V^\pi(s_{t+h}).
\end{equation}
Using the same $\pi$ for every $h$ makes these Q-values comparable across prefix lengths.
\end{theorem}

\subsection{Multi-Step TD Objective}\label{sec:td_objective}

We train the causal Transformer critic with a multi-step TD loss at every prefix length.
For each $h \in [H]$, define the $h$-step return target
\begin{equation}\label{eq:h_target}
    G_h(s_t, a_{t:t+h}) := \sum_{\tau=0}^{h-1} \gamma^\tau r_{t+\tau} + \gamma^{h}\, Q_{\bar\phi}\!\big(s_{t+h},\, \pi_\star(s_{t+h})\big),
\end{equation}
which sums the $h$ rewards obtained by executing the action prefix $a_{t:t+h}$, and a bootstrap value at $s_{t+h}$.
For the bootstrap value at $s_{t+h}$, $\pi_\star$ samples $N$ length-$H$ candidate chunks from $\pi_\theta$ and selects the highest-value prefix via joint rejection sampling over candidate index and prefix length (\cref{sec:adaptive_extraction}).
The critic loss averages the squared error against $G_h$ across $h \in [H]$ in expectation over chunked transitions from the offline dataset or replay buffer:
\begin{equation}\label{eq:critic_loss}
    \mathcal{L}(\phi) = \mathbb{E}_{(s_{t:t+H+1},\, a_{t:t+H},\, r_{t:t+H}) \sim \mathcal{D}}\!\left[\frac{1}{H} \sum_{h=1}^{H} \big( Q_\phi(s_t, a_{t:t+h}) - G_h(s_t, a_{t:t+h}) \big)^2\right].
\end{equation}
Each $G_h$ propagates the bootstrap value at $s_{t+h}$ back $h$ steps to $Q_\phi(s_t, a_{t:t+h})$ in a single critic update.
Averaging gradients across prefix lengths can further reduce gradient variance while preserving sparse reward signals, as formalized in \cref{app:theory-multi-horizon}~\citep{T-SAC}.

This TD backup generalizes EMaQ's expected-max Q operator~\citep{EMaQ} from a $\max$ over $N$ candidate actions to a joint $\max$ over $NH$ candidate prefixes, while retaining its contraction and convergence properties (\cref{thm:contraction}).

\begin{theorem}[Bellman contraction and convergence; proof in \cref{app:theory-bellman}]
\label{thm:contraction}
\label{thm:fixed_point_identity}
Fix the flow BC policy $\pi_\theta$, rejection sampling size $N$, and maximum chunk size $H$.
The resulting expected Bellman backup $\Bnh$ (\cref{def:bnh}) satisfies, for any bounded prefix-value functions $Q_1,Q_2$,
\begin{equation}
\label{eq:contraction}
\|\Bnh Q_1 - \Bnh Q_2\|_\infty \;\le\; \gamma\, \|Q_1 - Q_2\|_\infty.
\end{equation}
Repeated application of $\Bnh$ converges to its unique fixed point $\Qnh$ from any bounded initial Q-function.
Moreover, $\Qnh = Q^{\pi_\star}$, where the adaptive policy $\pi_\star$ selects the candidate prefix with the highest value under $\Qnh$.
\end{theorem}

These fixed-point Q-values lie on a common return scale across prefix lengths (\cref{thm:cross_horizon}), as required by the joint argmax over $(n, h)$ (\cref{sec:adaptive_extraction}).
We test the learned Q-values for calibration against realized returns and cross-horizon comparability in \cref{sec:qualitative-quantitative-analysis}.

\subsection{Adaptive Policy Extraction}\label{sec:adaptive_extraction}

We use a flow BC policy $\pi_\theta(s, z)$ to generate candidate action chunks of length $H$ from the offline dataset $\mathcal{D}$.
$\pi_\theta$ is parameterized by a state-conditioned velocity field $v_\theta: [0,1] \times \mathcal{S} \times \mathbb{R}^{Hd} \to \mathbb{R}^{Hd}$ trained with the flow-matching loss
\begin{equation}\label{eq:flow_loss}
    \mathcal{L}(\theta) = \mathbb{E}_{\substack{z \sim \mathcal{N}(0, I_{Hd}), \, (s_t, a_{t:t+H}) \sim \mathcal{D}, \, u \sim \text{Unif}([0,1])}} \left[ \| v_\theta(u, s_t, a^u_{t:t+H}) - (a_{t:t+H} - z) \|_2^2 \right],
\end{equation}
where $a^u_{t:t+H} := (1-u)z + ua_{t:t+H}$ is the linear interpolation between the noise $z$ and the action chunk $a_{t:t+H}$.
Fitting full-length action chunks lets $\pi_\theta$ capture non-Markovian behavior patterns in the offline data, supporting temporally coherent exploration in long-horizon tasks~\citep{QC}.

Given $\pi_\theta$ and the trained critic $Q_\phi$, we extract ACSAC's policy $\pi_\star$ via rejection sampling, 
which implicitly enforces a behavior constraint.
The KL divergence of the selected full-length chunk distribution from the flow BC policy distribution $\pi_\theta$ has a closed-form bound~\citep{QC}.
Prefix truncation does not further increase this KL divergence when the same rule is applied to both distributions (\cref{app:theory-truncation}).
Rejection sampling is also generally more robust to hyperparameters than alternatives that require tuning a behavior regularization coefficient~\citep{MAC}.

Prior action-chunking methods~\citep{QC, DQC, MAC} fix the chunk size $h$ and apply rejection sampling: $n^\star = \arg\max_{n \in [N]} Q_\phi(s_t, a^{(n)}_{t:t+h}),\quad \pi(s_t) = a^{(n^\star)}_{t:t+h}$.
ACSAC extends this single-axis rejection sampling to a joint search over $NH$ candidates indexed by $(n, h)$ (\cref{fig:pipeline}):
\begin{equation}\label{eq:adaptive_policy}
    (n^\star, h^\star) = \arg\max_{n \in [N],\, h \in [H]} Q_\phi\!\left(s_t, a^{(n)}_{t:t+h}\right),
    \qquad
    \pi_\star(s_t) = a^{(n^\star)}_{t:t+h^\star},
\end{equation}
where $\{a^{(n)}_{t:t+H}\}_{n \in [N]}$ are obtained by drawing $N$ noise vectors $z^{(n)} \sim \mathcal{N}(0, I_{Hd})$ and mapping them through the flow BC policy: $a^{(n)}_{t:t+H} = \pi_\theta(s_t, z^{(n)})$.
The Transformer critic provides Q-values for all $NH$ pairs, and $(n^\star, h^\star)$ is selected by a single argmax.

Intuitively, the selected execution horizon $h^\star$ reflects the current phase of the policy rollout at state $s_t$.
When the Q-values of longer prefixes decline relative to shorter ones, the agent should replan after fewer steps to maintain reactivity.
Conversely, when Q-values stay high or increase with longer prefixes, the state admits a coherent long-horizon plan and the agent benefits from a longer prefix to maintain temporal consistency.
The selected execution horizon $h^\star \in [H]$ thus varies per state, with shorter $h^\star$ at sensitive states and longer $h^\star$ during stable phases.

We provide pseudocode in \cref{alg:training} and further implementation details in \cref{app:impl-details}.

\begingroup
\makeatletter
% Apply the local adjustment after the bottom rule, even when the algorithm floats.
\let\acsac@algorithmstyle\fst@algorithm
\def\fst@algorithm{%
    \acsac@algorithmstyle
    \let\acsac@algorithmpost\@fs@post
    \def\@fs@post{\acsac@algorithmpost\vspace{-10pt}}%
}
\makeatother
\begin{algorithm}[t]
\caption{Adaptive Chunk Size Actor-Critic (ACSAC)}
\label{alg:training}
\small
\algrenewcommand{\algorithmicindent}{0.6em}
\begin{algorithmic}
    \Require Dataset $\mathcal{D}$, flow BC policy $\pi_\theta$, causal Transformer critic $Q_\phi$, maximum chunk size $H$, rejection sampling size $N$
\end{algorithmic}
\noindent
\begin{minipage}[t]{0.525\linewidth}
\begin{algorithmic}
    \State \blue{\emph{// Offline training loop}}
    \While{not converged}
        \State Sample $\{(s_{t:t+H+1}, a_{t:t+H}, r_{t:t+H})\} \sim \mathcal{D}$.
        \State Update $\pi_\theta$ with the flow-matching loss (\cref{eq:flow_loss}).
        \State Update $Q_\phi$ with the multi-step TD loss (\cref{eq:critic_loss}).
    \EndWhile
    \State \blue{\emph{// Adaptive policy extraction with rejection sampling}}
    \Function{$\pi_\star$}{$s_t$}
        \State $z^{(n)} \sim \mathcal{N}(0, I_{Hd}),\ a^{(n)}_{t:t+H} = \pi_\theta(s_t, z^{(n)}),\ n \in [N]$.
        \State $(n^\star, h^\star) \gets \arg\max_{n \in [N],\ h \in [H]} Q_\phi(s_t, a^{(n)}_{t:t+h})$
        \State \Return $a^{(n^\star)}_{t:t+h^\star}$
    \EndFunction
\end{algorithmic}
\end{minipage}\hfill
\begin{minipage}[t]{0.475\linewidth}
\begin{algorithmic}
    \State \blue{\emph{// Online fine-tuning with adaptive replanning}}
    \State Initialize $\mathcal{D}$ with offline data.
    \For{every environment step $t$}
        \If{the previous selected chunk has been\\ \hspace*{0.5em}fully executed}
            \State $a^\star_{t:t+h^\star} \leftarrow \pi_\star(s_t)$
        \EndIf
        \State Act with $a^\star_t$ and receive $s_{t+1}, r_t$.
        \State $\mathcal{D} \leftarrow \mathcal{D} \cup \{(s_t, a^\star_t, s_{t+1}, r_t)\}$.
        \State Sample $\{(s_{t:t+H+1}, a_{t:t+H}, r_{t:t+H})\} \sim \mathcal{D}$.
        \State Update $\pi_\theta$ with the flow-matching loss (\cref{eq:flow_loss}).
        \State Update $Q_\phi$ with the multi-step TD loss (\cref{eq:critic_loss}).
    \EndFor
\end{algorithmic}
\end{minipage}
\end{algorithm}
\endgroup

\vspace{-7pt}
\section{Experiments}
\label{sec:experiments}
\vspace{-3pt}

We evaluate ACSAC on manipulation and navigation tasks from OGBench~\citep{OGBench}.
Our experiments are designed to answer three questions.
\textbf{Q1.} Does ACSAC improve performance in both offline RL and offline-to-online RL over prior methods?
\textbf{Q2.} Are ACSAC's prefix-conditioned Q-values calibrated and cross-horizon comparable, and does its adaptive chunk size selection follow task phases?
\textbf{Q3.} How do the maximum chunk size $H$ and the rejection sampling size $N$ affect performance, and does ACSAC's advantage come from the causal Transformer critic alone?

\vspace{-6pt}
\subsection{Experimental Setup}
\label{sec:exp-setup}
\vspace{-4pt}

\textbf{Environments and datasets.}
We evaluate ACSAC on eight long-horizon, sparse-reward domains from OGBench (five tasks each, $40$ tasks in total), with domain abbreviations following QAM~\citep{QAM}.
The five manipulation domains are \texttt{scene-sparse} (\texttt{scene}), \texttt{puzzle-3x3-sparse} (\texttt{p33}), \texttt{cube-double} (\texttt{c2}), \texttt{cube-triple} (\texttt{c3}), and \texttt{cube-quadruple} (\texttt{c4}).
The three navigation domains are \texttt{antmaze-giant} (\texttt{ag}), \texttt{antsoccer-arena} (\texttt{soccer}), and \texttt{humanoidmaze-medium} (\texttt{hm}).
We use the standard \texttt{play} datasets for manipulation and \texttt{navigate} datasets for navigation~\citep{OGBench}.
For \texttt{cube-quadruple}, we use the $100$M-size dataset~\citep{QC}.
All methods use $4$ random seeds per task in the main experiments.
Additional domain metadata is provided in \cref{app:domain}.

\textbf{Baselines.}
We compare ACSAC against single-step, multi-step, and fixed chunk size methods.
The single-step baselines are \texttt{IQL}~\citep{IQL}, \texttt{ReBRAC}~\citep{ReBRAC}, \texttt{FQL}~\citep{FQL}, and \texttt{BFN}~\citep{QC}.
The multi-step baselines are \texttt{FQL-n} and \texttt{BFN-n}~\citep{QC}, which use multi-step returns without learning both a chunked critic and a chunked policy.
Finally, we compare against fixed chunk size methods, including \texttt{QC}, \texttt{QC-FQL}~\citep{QC}, and \texttt{DEAS}~\citep{DEAS}.
Implementation details for all baselines are summarized in \cref{app:baselines}.

\vspace{-10pt}
\begin{table}[H]
\renewcommand*{\ttdefault}{cmtt}
\caption{\textbf{Summary table for OGBench offline-to-online RL results.} Each cell reports mean success rates (\%) over $4$ seeds at the end of $1$M offline training steps $\rightarrow$ $1$M additional online steps, with $95\%$ confidence intervals above each mean. Manipulation and navigation are averaged separately. The best nonzero offline and online results in each row are highlighted in color.}
\label{tab:multitask}
\begin{center}
\renewcommand{\arraystretch}{1.05}
\setlength{\extrarowheight}{1.5pt}
\setlength{\tabcolsep}{6pt}
\ifdefined\perftablebox\else\newsavebox{\perftablebox}\fi
\newcommand{\summaryresult}[2]{\overset{\text{\fontsize{5}{5}\selectfont[#1]}}{\textbf{#2}}}
\sbox{\perftablebox}{%
\begin{tabular}{l>{\columncolor[HTML]{F5FAFE}}c>{\columncolor[HTML]{F5FAFE}}c>{\columncolor[HTML]{F5FAFE}}c>{\columncolor[HTML]{F5FAFE}}c>{\columncolor[HTML]{F7F6F9}}c>{\columncolor[HTML]{F7F6F9}}c>{\columncolor[HTML]{FDF7FA}}c>{\columncolor[HTML]{FDF7FA}}c>{\columncolor[HTML]{FDF7FA}}c>{\columncolor[HTML]{F2FBFA}}c}
\toprule
 & \multicolumn{4}{c}{Single-step} & \multicolumn{2}{c}{Multi-step} & \multicolumn{3}{c}{Fixed Chunk Size} & \multicolumn{1}{c}{Adaptive} \\
 & \texttt{IQL} & \texttt{ReBRAC} & \texttt{FQL} & \texttt{BFN} & \texttt{FQL-n} & \texttt{BFN-n} & \texttt{QC-FQL} & \texttt{QC} & \texttt{DEAS} & {\fontfamily{lmtt}\bfseries ACSAC} \\
\midrule
\multicolumn{11}{l}{\textbf{Manipulation (25 tasks)}} \\
\texttt{p33} & \ensuremath{{\color{gray!60}\summaryresult{0,0}{0}}\mathord{\rightarrow}{\color{gray!60}\summaryresult{19,20}{20}}} & \ensuremath{{\color{gray!60}\summaryresult{58,67}{62}}\mathord{\rightarrow}{\color[HTML]{0076BA}\summaryresult{100,100}{100}}} & \ensuremath{{\color{gray!60}\summaryresult{98,100}{99}}\mathord{\rightarrow}{\color[HTML]{0076BA}\summaryresult{100,100}{100}}} & \ensuremath{{\color{gray!60}\summaryresult{95,98}{97}}\mathord{\rightarrow}{\color[HTML]{0076BA}\summaryresult{100,100}{100}}} & \ensuremath{{\color{gray!60}\summaryresult{99,100}{99}}\mathord{\rightarrow}{\color[HTML]{66509B}\summaryresult{100,100}{100}}} & \ensuremath{{\color{gray!60}\summaryresult{57,59}{58}}\mathord{\rightarrow}{\color{gray!60}\summaryresult{89,92}{90}}} & \ensuremath{{\color{gray!60}\summaryresult{50,76}{64}}\mathord{\rightarrow}{\color[HTML]{CB297B}\summaryresult{100,100}{100}}} & \ensuremath{{\color[HTML]{CB297B}\summaryresult{100,100}{100}}\mathord{\rightarrow}{\color[HTML]{CB297B}\summaryresult{100,100}{100}}} & \ensuremath{{\color{gray!60}\summaryresult{93,97}{95}}\mathord{\rightarrow}{\color[HTML]{CB297B}\summaryresult{100,100}{100}}} & \ensuremath{{\color[HTML]{00A89D}\summaryresult{100,100}{100}}\mathord{\rightarrow}{\color[HTML]{00A89D}\summaryresult{100,100}{100}}} \\
\texttt{scene} & \ensuremath{{\color{gray!60}\summaryresult{0,0}{0}}\mathord{\rightarrow}{\color{gray!60}\summaryresult{39,39}{39}}} & \ensuremath{{\color{gray!60}\summaryresult{11,19}{15}}\mathord{\rightarrow}{\color{gray!60}\summaryresult{99,99}{99}}} & \ensuremath{{\color{gray!60}\summaryresult{43,61}{52}}\mathord{\rightarrow}{\color{gray!60}\summaryresult{93,98}{95}}} & \ensuremath{{\color{gray!60}\summaryresult{84,88}{87}}\mathord{\rightarrow}{\color{gray!60}\summaryresult{99,100}{99}}} & \ensuremath{{\color{gray!60}\summaryresult{17,26}{21}}\mathord{\rightarrow}{\color{gray!60}\summaryresult{64,76}{70}}} & \ensuremath{{\color{gray!60}\summaryresult{53,58}{55}}\mathord{\rightarrow}{\color{gray!60}\summaryresult{95,99}{97}}} & \ensuremath{{\color{gray!60}\summaryresult{83,89}{86}}\mathord{\rightarrow}{\color{gray!60}\summaryresult{99,100}{99}}} & \ensuremath{{\color{gray!60}\summaryresult{79,85}{82}}\mathord{\rightarrow}{\color{gray!60}\summaryresult{98,100}{99}}} & \ensuremath{{\color{gray!60}\summaryresult{91,94}{92}}\mathord{\rightarrow}{\color{gray!60}\summaryresult{93,98}{95}}} & \ensuremath{{\color[HTML]{00A89D}\summaryresult{99,100}{100}}\mathord{\rightarrow}{\color[HTML]{00A89D}\summaryresult{99,100}{100}}} \\
\texttt{c2} & \ensuremath{{\color{gray!60}\summaryresult{0,0}{0}}\mathord{\rightarrow}{\color{gray!60}\summaryresult{0,0}{0}}} & \ensuremath{{\color{gray!60}\summaryresult{2,4}{3}}\mathord{\rightarrow}{\color{gray!60}\summaryresult{28,32}{30}}} & \ensuremath{{\color{gray!60}\summaryresult{29,36}{32}}\mathord{\rightarrow}{\color{gray!60}\summaryresult{75,76}{76}}} & \ensuremath{{\color{gray!60}\summaryresult{65,69}{67}}\mathord{\rightarrow}{\color{gray!60}\summaryresult{78,81}{79}}} & \ensuremath{{\color{gray!60}\summaryresult{5,13}{9}}\mathord{\rightarrow}{\color{gray!60}\summaryresult{76,77}{77}}} & \ensuremath{{\color{gray!60}\summaryresult{5,7}{6}}\mathord{\rightarrow}{\color{gray!60}\summaryresult{62,69}{65}}} & \ensuremath{{\color{gray!60}\summaryresult{39,47}{42}}\mathord{\rightarrow}{\color[HTML]{CB297B}\summaryresult{100,100}{100}}} & \ensuremath{{\color{gray!60}\summaryresult{63,70}{66}}\mathord{\rightarrow}{\color{gray!60}\summaryresult{96,99}{98}}} & \ensuremath{{\color{gray!60}\summaryresult{52,60}{56}}\mathord{\rightarrow}{\color{gray!60}\summaryresult{84,87}{86}}} & \ensuremath{{\color[HTML]{00A89D}\summaryresult{84,86}{85}}\mathord{\rightarrow}{\color[HTML]{00A89D}\summaryresult{99,100}{100}}} \\
\texttt{c3} & \ensuremath{{\color{gray!60}\summaryresult{0,0}{0}}\mathord{\rightarrow}{\color{gray!60}\summaryresult{0,0}{0}}} & \ensuremath{{\color{gray!60}\summaryresult{0,0}{0}}\mathord{\rightarrow}{\color{gray!60}\summaryresult{0,0}{0}}} & \ensuremath{{\color{gray!60}\summaryresult{0,3}{2}}\mathord{\rightarrow}{\color{gray!60}\summaryresult{16,20}{18}}} & \ensuremath{{\color{gray!60}\summaryresult{4,8}{6}}\mathord{\rightarrow}{\color{gray!60}\summaryresult{21,28}{23}}} & \ensuremath{{\color{gray!60}\summaryresult{0,2}{1}}\mathord{\rightarrow}{\color{gray!60}\summaryresult{0,1}{1}}} & \ensuremath{{\color{gray!60}\summaryresult{0,1}{0}}\mathord{\rightarrow}{\color{gray!60}\summaryresult{0,1}{0}}} & \ensuremath{{\color{gray!60}\summaryresult{2,4}{3}}\mathord{\rightarrow}{\color{gray!60}\summaryresult{46,61}{53}}} & \ensuremath{{\color{gray!60}\summaryresult{4,6}{5}}\mathord{\rightarrow}{\color{gray!60}\summaryresult{61,66}{64}}} & \ensuremath{{\color{gray!60}\summaryresult{3,6}{4}}\mathord{\rightarrow}{\color{gray!60}\summaryresult{67,71}{69}}} & \ensuremath{{\color[HTML]{00A89D}\summaryresult{15,19}{17}}\mathord{\rightarrow}{\color[HTML]{00A89D}\summaryresult{97,99}{98}}} \\
\texttt{c4} & \ensuremath{{\color{gray!60}\summaryresult{0,0}{0}}\mathord{\rightarrow}{\color{gray!60}\summaryresult{0,0}{0}}} & \ensuremath{{\color{gray!60}\summaryresult{0,1}{1}}\mathord{\rightarrow}{\color{gray!60}\summaryresult{20,20}{20}}} & \ensuremath{{\color{gray!60}\summaryresult{0,0}{0}}\mathord{\rightarrow}{\color{gray!60}\summaryresult{0,6}{3}}} & \ensuremath{{\color{gray!60}\summaryresult{0,4}{2}}\mathord{\rightarrow}{\color{gray!60}\summaryresult{12,13}{12}}} & \ensuremath{{\color{gray!60}\summaryresult{6,9}{7}}\mathord{\rightarrow}{\color{gray!60}\summaryresult{36,37}{36}}} & \ensuremath{{\color{gray!60}\summaryresult{0,0}{0}}\mathord{\rightarrow}{\color{gray!60}\summaryresult{0,0}{0}}} & \ensuremath{{\color{gray!60}\summaryresult{1,3}{2}}\mathord{\rightarrow}{\color[HTML]{CB297B}\summaryresult{76,77}{77}}} & \ensuremath{{\color{gray!60}\summaryresult{2,4}{3}}\mathord{\rightarrow}{\color{gray!60}\summaryresult{72,75}{73}}} & \ensuremath{{\color{gray!60}\summaryresult{0,1}{0}}\mathord{\rightarrow}{\color{gray!60}\summaryresult{2,16}{9}}} & \ensuremath{{\color[HTML]{00A89D}\summaryresult{5,12}{9}}\mathord{\rightarrow}{\color{gray!60}\summaryresult{72,75}{73}}} \\
\texttt{Overall} & \ensuremath{{\color{gray!60}\summaryresult{0,0}{0}}\mathord{\rightarrow}{\color{gray!60}\summaryresult{12,12}{12}}} & \ensuremath{{\color{gray!60}\summaryresult{15,18}{16}}\mathord{\rightarrow}{\color{gray!60}\summaryresult{49,50}{50}}} & \ensuremath{{\color{gray!60}\summaryresult{36,38}{37}}\mathord{\rightarrow}{\color{gray!60}\summaryresult{57,60}{58}}} & \ensuremath{{\color{gray!60}\summaryresult{51,53}{52}}\mathord{\rightarrow}{\color{gray!60}\summaryresult{62,63}{63}}} & \ensuremath{{\color{gray!60}\summaryresult{26,29}{28}}\mathord{\rightarrow}{\color{gray!60}\summaryresult{56,58}{57}}} & \ensuremath{{\color{gray!60}\summaryresult{24,25}{24}}\mathord{\rightarrow}{\color{gray!60}\summaryresult{48,52}{50}}} & \ensuremath{{\color{gray!60}\summaryresult{35,43}{39}}\mathord{\rightarrow}{\color{gray!60}\summaryresult{84,88}{86}}} & \ensuremath{{\color{gray!60}\summaryresult{50,52}{51}}\mathord{\rightarrow}{\color{gray!60}\summaryresult{86,87}{86}}} & \ensuremath{{\color{gray!60}\summaryresult{49,51}{50}}\mathord{\rightarrow}{\color{gray!60}\summaryresult{70,74}{72}}} & \ensuremath{{\color[HTML]{00A89D}\summaryresult{61,63}{62}}\mathord{\rightarrow}{\color[HTML]{00A89D}\summaryresult{94,95}{94}}} \\
\midrule
\multicolumn{11}{l}{\textbf{Navigation (15 tasks)}} \\
\texttt{ag} & \ensuremath{{\color{gray!60}\summaryresult{2,3}{2}}\mathord{\rightarrow}{\color{gray!60}\summaryresult{1,3}{2}}} & \ensuremath{{\color[HTML]{0076BA}\summaryresult{44,50}{47}}\mathord{\rightarrow}{\color{gray!60}\summaryresult{58,66}{62}}} & \ensuremath{{\color{gray!60}\summaryresult{3,11}{6}}\mathord{\rightarrow}{\color{gray!60}\summaryresult{64,79}{74}}} & \ensuremath{{\color{gray!60}\summaryresult{0,2}{0}}\mathord{\rightarrow}{\color{gray!60}\summaryresult{16,23}{20}}} & \ensuremath{{\color{gray!60}\summaryresult{19,25}{22}}\mathord{\rightarrow}{\color{gray!60}\summaryresult{70,74}{72}}} & \ensuremath{{\color{gray!60}\summaryresult{17,24}{21}}\mathord{\rightarrow}{\color{gray!60}\summaryresult{66,70}{68}}} & \ensuremath{{\color{gray!60}\summaryresult{0,0}{0}}\mathord{\rightarrow}{\color{gray!60}\summaryresult{49,62}{56}}} & \ensuremath{{\color{gray!60}\summaryresult{0,0}{0}}\mathord{\rightarrow}{\color{gray!60}\summaryresult{62,79}{70}}} & \ensuremath{{\color{gray!60}\summaryresult{1,1}{1}}\mathord{\rightarrow}{\color{gray!60}\summaryresult{18,20}{19}}} & \ensuremath{{\color{gray!60}\summaryresult{0,0}{0}}\mathord{\rightarrow}{\color[HTML]{00A89D}\summaryresult{86,89}{87}}} \\
\texttt{soccer} & \ensuremath{{\color{gray!60}\summaryresult{7,11}{9}}\mathord{\rightarrow}{\color{gray!60}\summaryresult{2,5}{4}}} & \ensuremath{{\color{gray!60}\summaryresult{0,0}{0}}\mathord{\rightarrow}{\color{gray!60}\summaryresult{0,0}{0}}} & \ensuremath{{\color{gray!60}\summaryresult{29,39}{34}}\mathord{\rightarrow}{\color{gray!60}\summaryresult{70,74}{72}}} & \ensuremath{{\color{gray!60}\summaryresult{39,45}{42}}\mathord{\rightarrow}{\color{gray!60}\summaryresult{64,74}{68}}} & \ensuremath{{\color{gray!60}\summaryresult{5,9}{7}}\mathord{\rightarrow}{\color{gray!60}\summaryresult{5,13}{9}}} & \ensuremath{{\color{gray!60}\summaryresult{12,16}{14}}\mathord{\rightarrow}{\color{gray!60}\summaryresult{19,25}{22}}} & \ensuremath{{\color{gray!60}\summaryresult{6,10}{8}}\mathord{\rightarrow}{\color{gray!60}\summaryresult{64,69}{67}}} & \ensuremath{{\color{gray!60}\summaryresult{31,34}{32}}\mathord{\rightarrow}{\color{gray!60}\summaryresult{80,84}{82}}} & \ensuremath{{\color{gray!60}\summaryresult{3,5}{4}}\mathord{\rightarrow}{\color{gray!60}\summaryresult{68,74}{71}}} & \ensuremath{{\color[HTML]{00A89D}\summaryresult{50,54}{52}}\mathord{\rightarrow}{\color[HTML]{00A89D}\summaryresult{92,94}{93}}} \\
\texttt{hm} & \ensuremath{{\color{gray!60}\summaryresult{29,35}{32}}\mathord{\rightarrow}{\color{gray!60}\summaryresult{13,18}{16}}} & \ensuremath{{\color{gray!60}\summaryresult{14,30}{22}}\mathord{\rightarrow}{\color{gray!60}\summaryresult{0,1}{1}}} & \ensuremath{{\color{gray!60}\summaryresult{31,37}{34}}\mathord{\rightarrow}{\color{gray!60}\summaryresult{54,64}{59}}} & \ensuremath{{\color{gray!60}\summaryresult{46,61}{54}}\mathord{\rightarrow}{\color{gray!60}\summaryresult{73,95}{84}}} & \ensuremath{{\color{gray!60}\summaryresult{31,37}{34}}\mathord{\rightarrow}{\color{gray!60}\summaryresult{52,57}{54}}} & \ensuremath{{\color{gray!60}\summaryresult{22,26}{24}}\mathord{\rightarrow}{\color{gray!60}\summaryresult{35,42}{38}}} & \ensuremath{{\color{gray!60}\summaryresult{11,19}{15}}\mathord{\rightarrow}{\color{gray!60}\summaryresult{20,32}{26}}} & \ensuremath{{\color{gray!60}\summaryresult{31,36}{34}}\mathord{\rightarrow}{\color{gray!60}\summaryresult{70,75}{73}}} & \ensuremath{{\color{gray!60}\summaryresult{45,51}{48}}\mathord{\rightarrow}{\color{gray!60}\summaryresult{59,65}{62}}} & \ensuremath{{\color[HTML]{00A89D}\summaryresult{65,68}{66}}\mathord{\rightarrow}{\color[HTML]{00A89D}\summaryresult{90,94}{92}}} \\
\texttt{Overall} & \ensuremath{{\color{gray!60}\summaryresult{13,16}{15}}\mathord{\rightarrow}{\color{gray!60}\summaryresult{6,8}{7}}} & \ensuremath{{\color{gray!60}\summaryresult{20,26}{23}}\mathord{\rightarrow}{\color{gray!60}\summaryresult{20,22}{21}}} & \ensuremath{{\color{gray!60}\summaryresult{22,27}{25}}\mathord{\rightarrow}{\color{gray!60}\summaryresult{65,71}{68}}} & \ensuremath{{\color{gray!60}\summaryresult{30,35}{32}}\mathord{\rightarrow}{\color{gray!60}\summaryresult{53,62}{57}}} & \ensuremath{{\color{gray!60}\summaryresult{19,22}{21}}\mathord{\rightarrow}{\color{gray!60}\summaryresult{43,47}{45}}} & \ensuremath{{\color{gray!60}\summaryresult{18,21}{20}}\mathord{\rightarrow}{\color{gray!60}\summaryresult{41,44}{43}}} & \ensuremath{{\color{gray!60}\summaryresult{6,9}{8}}\mathord{\rightarrow}{\color{gray!60}\summaryresult{46,53}{49}}} & \ensuremath{{\color{gray!60}\summaryresult{21,23}{22}}\mathord{\rightarrow}{\color{gray!60}\summaryresult{72,78}{75}}} & \ensuremath{{\color{gray!60}\summaryresult{16,19}{18}}\mathord{\rightarrow}{\color{gray!60}\summaryresult{49,52}{51}}} & \ensuremath{{\color[HTML]{00A89D}\summaryresult{38,40}{39}}\mathord{\rightarrow}{\color[HTML]{00A89D}\summaryresult{90,92}{91}}} \\
\bottomrule
\end{tabular}%
}
\ifdim\wd\perftablebox>\textwidth
  \sbox{\perftablebox}{\resizebox{\textwidth}{!}{\usebox{\perftablebox}}}%
\fi
\makebox[\textwidth]{\usebox{\perftablebox}}
\end{center}
\end{table}

\subsection{Main Results}
\label{sec:main-results}

We report the offline and offline-to-online RL results in \cref{tab:multitask}.
ACSAC achieves the highest average success rates on both manipulation and navigation in both settings.

\textbf{Offline RL.}
ACSAC reaches $62\%$ average success on manipulation and $39\%$ on navigation, compared with $52\%$ and $32\%$ for \texttt{BFN}, the strongest baseline in both task families.
The gains are particularly clear on \texttt{cube-double}, where ACSAC achieves $85\%$ success versus $67\%$ for \texttt{BFN}, and \texttt{humanoidmaze-medium}, where it reaches $66\%$ versus $54\%$ for \texttt{BFN}.

\textbf{Offline-to-online RL.}
After $1$M additional environment steps, ACSAC reaches $94\%$ average success on manipulation and $91\%$ on navigation, compared with $86\%$ and $75\%$ for \texttt{QC}.
The largest manipulation gain occurs on \texttt{cube-triple}, where ACSAC achieves $98\%$ success versus $64\%$ for \texttt{QC} and $69\%$ for \texttt{DEAS}.
On navigation, ACSAC achieves the highest mean success rate on all three domains, extending its advantage to tasks requiring reactive control.
ACSAC remains competitive on \texttt{cube-quadruple}, matching \texttt{QC} at $73\%$ online success, compared with $77\%$ for \texttt{QC-FQL}.
Complete per-task results and learning curves are provided in \cref{app:full-o2o-results}, with additional ACSAC results over $8$ random seeds in \cref{app:more-seeds}.

\vspace{-6pt}
\subsection{Qualitative and Quantitative Analyses}
\label{sec:qualitative-quantitative-analysis}

\noindent
\begin{minipage}[b]{\textwidth}
\begin{minipage}[t]{0.60\textwidth}
\vspace{0pt}
\textbf{Prefix-Q calibration and cross-horizon comparability.}
We test whether ACSAC's prefix-conditioned Q-values are calibrated against realized returns and comparable across horizons.
On \texttt{cube-triple-task4}, we collect $50$ rollouts each for the adaptive policy and five fixed-$h$ variants ($h=1,\ldots,5$), using the checkpoint after $1$M steps of online fine-tuning.
Each fixed-$h$ variant executes the highest-value length-$h$ prefix from $N$ candidate chunks and then follows $\pi_\star$.
At each replanning state, we record the predicted Q-value $\hat Q$ of the selected prefix and the discounted Monte Carlo return $\hat G$ realized from that state.
For each policy, we sort these pairs by $\hat G$, divide them into ten equal-count bins, and plot the mean $\hat Q$ against the mean $\hat G$ in each bin.
As shown in \cref{fig:q_mc}, the six curves broadly follow the diagonal $\hat Q=\hat G$ and largely coincide, indicating approximate calibration and a common return scale across horizons.
\end{minipage}\hfill
\begin{minipage}[t]{0.38\textwidth}
    \vspace{-2pt}
    \centering
    \setlength{\abovecaptionskip}{0pt}
    \includegraphics[width=0.94\linewidth,trim=0 6 0 0,clip]{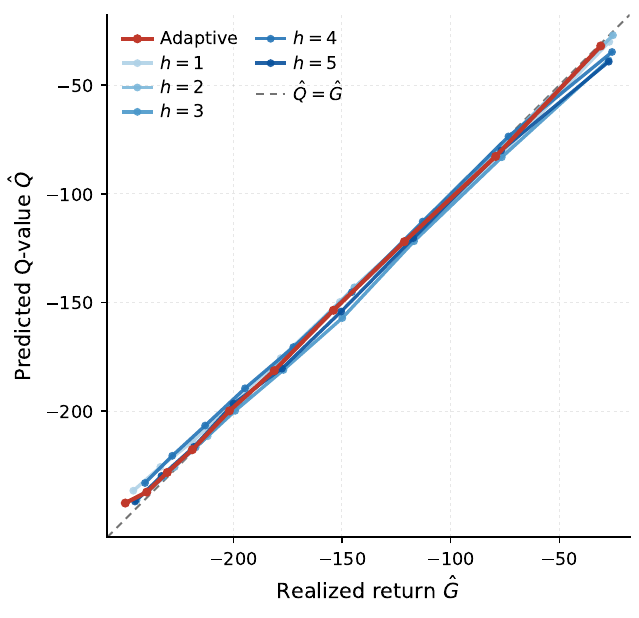}
    \par\vspace{2pt}
    \captionof{figure}{\textbf{Prefix-Q calibration.}
    Points show mean $\hat Q$ and mean $\hat G$ within each bin.
    The dashed line denotes $\hat Q=\hat G$.}
    \label{fig:q_mc}
\end{minipage}
\par\kern0pt
\end{minipage}

\vspace{-10pt}
\begin{figure}[H]
    \centering
    \includegraphics[width=1\linewidth]{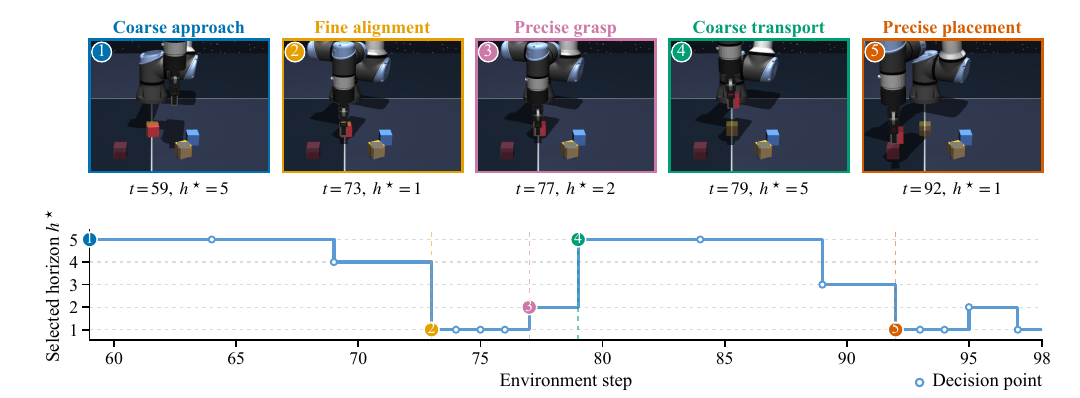}
    \vspace{-10pt}
    \caption{\textbf{Qualitative rollout of ACSAC.}
    Selected execution horizon $h^\star$ during a pick-and-place segment on \texttt{cube-triple-task4}.
    The numbered frames illustrate the corresponding task phases, with labels $t$ and $h^\star$ referring to the marked replanning decisions below.}
    \label{fig:h_dist}
    \vspace{-10pt}
\end{figure}

\textbf{State-dependent chunk size selection.}
\cref{fig:h_dist} visualizes the selected execution horizon $h^\star$ during a pick-and-place segment of an evaluation rollout on \texttt{cube-triple-task4}.
At the highlighted decisions, ACSAC selects a long chunk for the coarse approach ($h^\star=5$) and shorter chunks for fine alignment ($h^\star=1$) and precise grasping ($h^\star=2$).
It then returns to a longer horizon for coarse transport ($h^\star=5$) and shortens it again as precise placement begins ($h^\star=1$).
This pattern illustrates how maximizing the learned prefix-conditioned Q-values adjusts the replanning frequency across task phases, combining longer open-loop execution during coarse motion with frequent feedback during precise manipulation.
\cref{app:larger-chunk-size} presents further results with a larger maximum chunk size of $H=10$.

\subsection{Ablation Study}
\label{sec:ablations}

We conduct ablation studies on the \texttt{cube-double} and \texttt{cube-triple} domains, using $4$ random seeds per task.

\begin{figure}[htbp]
    \centering
    \includegraphics[width=1\linewidth,trim=14 6 1 24,clip]{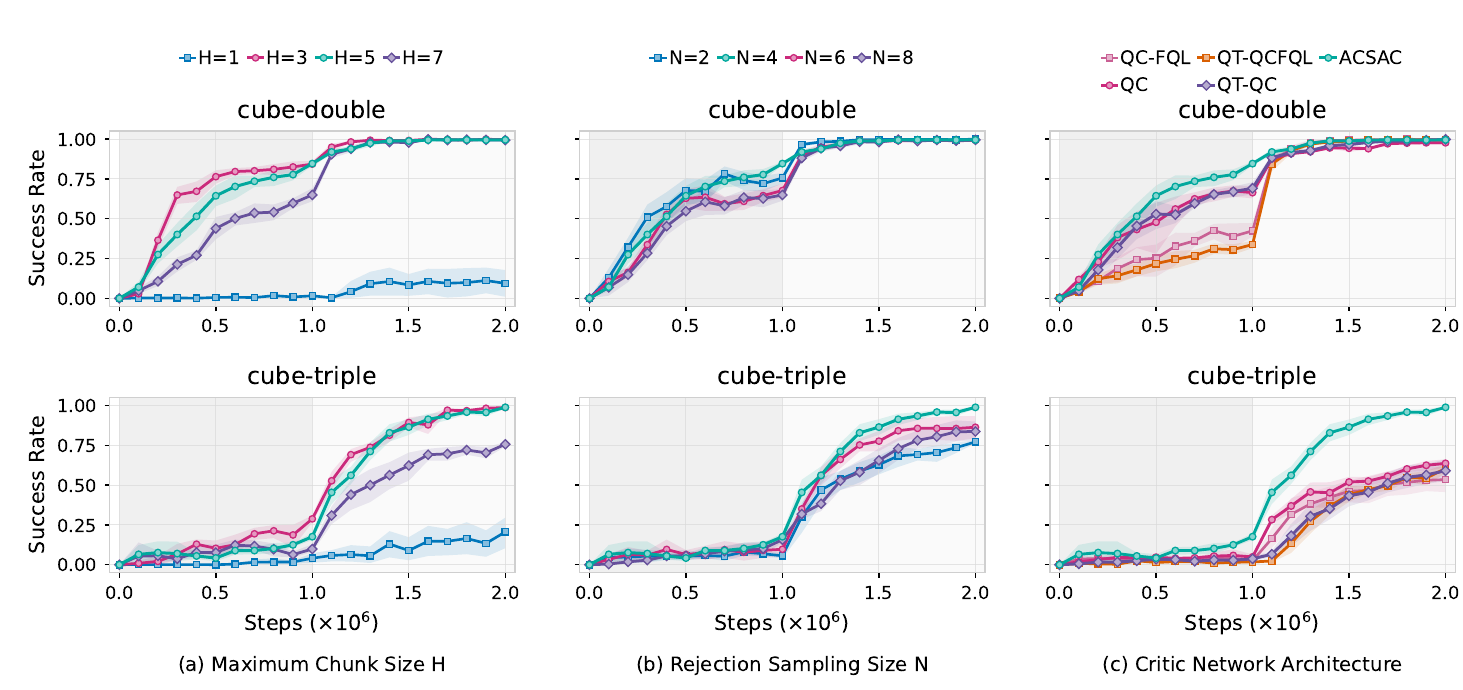}
    \vspace{-10pt}
    \caption{\textbf{Ablation studies ($4$ seeds).}
    \emph{(a):} maximum chunk size $H$ sweep ($N=4$).
    \emph{(b):} rejection sampling size $N$ sweep ($H=5$).
    \emph{(c):} QT-QC and QT-QCFQL replace the MLP critics in QC and QC-FQL with ACSAC's causal Transformer critic.
    Curves aggregate five tasks per domain over $4$ random seeds, with $95\%$ confidence intervals.
    The first $1$M steps are offline and the next $1$M steps are online.
    }
    \label{fig:ablations}
\end{figure}

In \cref{fig:ablations} (a,b), we examine the effects of the maximum chunk size $H \in \{1, 3, 5, 7\}$ and rejection sampling size $N \in \{2, 4, 6, 8\}$.
$H{=}1$ performs poorly on both domains, highlighting the importance of multi-step value propagation for these sparse-reward tasks.
$H{=}3$ and $H{=}5$ achieve similarly strong final online performance, whereas $H{=}7$ reduces performance on \texttt{cube-triple}.
We suspect that longer chunks make the behavior policy and critic harder to learn, as they require modeling longer action sequences and estimating more prefix values with the same network capacity~\citep{QC, DEAS}.
$N{=}4$ achieves the best online performance on \texttt{cube-triple}, suggesting a favorable balance between candidate coverage and maximization bias.
We use $H{=}5$ and $N{=}4$ in most domains, as this configuration performs well while providing a broad range of execution horizons with a modest sampling budget (\cref{app:methods-hyperparameters}).

In \cref{fig:ablations} (c), we replace the MLP critics in QC and QC-FQL with ACSAC's causal Transformer critic, yielding QT-QC and QT-QCFQL while keeping the remaining components unchanged.
On \texttt{cube-triple}, neither replacement consistently improves upon its corresponding baseline, and both remain substantially below ACSAC.
These results show that ACSAC's gains extend beyond the critic backbone.
\cref{app:redeploy} further isolates the benefit of adaptive chunk size selection with the trained model held fixed.

\section{Conclusion}
\label{sec:conclusion}
We propose ACSAC, which learns calibrated and cross-horizon comparable prefix-conditioned Q-values, supporting adaptive replanning without manually fixing the execution horizon.
A causal Transformer critic trained with a multi-step TD objective produces these values, and a flow BC policy paired with rejection sampling over the joint $(n, h)$ axes yields the executed chunk.
We prove that the ACSAC Bellman operator is a $\gamma$-contraction whose unique fixed point is the action-value function of the adaptive policy.
Experiments on OGBench demonstrate that ACSAC achieves strong performance on long-horizon, sparse-reward manipulation and navigation tasks across both offline RL and offline-to-online RL settings.
A promising future direction is to extend ACSAC's state-dependent chunk size selection to vision-language-action models for real-world robotic deployment.

\subsection*{AI use statement}

We used large language models to refine and polish our writing.
The authors bear full responsibility for the content and presentation of this paper.

\subsection*{Reproducibility statement}

We include additional implementation details for hyperparameters, datasets, and evaluation protocols in \cref{app:impl-details,app:experiment-details}.
The assumptions and proofs of all theoretical results are provided in \cref{app:theory}.
Our source code will be made public upon publication.

\bibliographystyle{iclr2027_conference}
\bibliography{iclr2027_conference}

\newpage
\appendix
\crefalias{section}{appendix}
\crefalias{subsection}{subappendix}
\numberwithin{theorem}{section}
\renewcommand{\theHtable}{\thesection.\arabic{table}}
\renewcommand{\theHfigure}{\thesection.\arabic{figure}}
\section{Limitations}
\label{app:limitations}

We highlight three limitations of ACSAC and corresponding directions for future work.
First, our online fine-tuning follows a simple protocol that adds online transitions to the offline dataset while retaining the same training objective.
How adaptive chunk selection interacts with specialized replay strategies, such as balanced sampling~\citep{Off2On}, remains unexplored.
Second, while \cref{prop:stochastic_transitions} establishes contraction and fixed-point results for stochastic transitions, a more comprehensive theoretical analysis of learning in stochastic environments is left for future work.
Third, extending ACSAC to large vision-language-action (VLA) models remains an open direction.
Since ACSAC selects among candidates from a behavior policy without separately parameterizing an RL policy~\citep{QC}, a possible extension is to use a frozen pretrained VLA to generate candidate chunks and train a prefix-conditioned critic for adaptive chunk selection.

\section{Implementation Details}
\label{app:impl-details}

In this section, we describe the full implementation details of ACSAC.

\paragraph{Online fine-tuning.}
For offline-to-online RL, we simply add online transitions to the dataset, without distinguishing them from the offline transitions.
We continue to train ACSAC with the same objective as in offline training, following~\citet{FQL} and~\citet{QC}.

\paragraph{Flow matching.}
Following~\citet{QC}, ACSAC trains a flow BC policy on the dataset, parameterized by a state-conditioned velocity field $v_\theta$ trained with the flow-matching loss in \cref{eq:flow_loss}.
At inference, an action chunk is generated by Euler integration of $v_\theta$ over $F$ flow steps from $z \sim \mathcal{N}(0, I_{Hd})$.
We use a step count of $10$ for the Euler method across all tasks.

\paragraph{Causal Transformer.}
ACSAC ingests $(s_t, a_t, a_{t+1}, \ldots, a_{t+H-1})$ as a token sequence and outputs $H$ prefix values, $[Q_\phi(s_t, a_{t}), Q_\phi(s_t, a_{t:t+2}), \ldots, Q_\phi(s_t, a_{t:t+H})]$, in one forward pass.
A causal attention mask ensures that the $h$-th prefix value depends only on $s_t$ and $a_{t:t+h}$.
Each block consists of multi-head self-attention followed by a two-layer feed-forward network with a hidden dimension of $512$.
Pre-LayerNorm is applied before every attention and feed-forward sublayer.
The Transformer critic uses $2$ attention layers, $8$ attention heads with per-head dimension $16$, and embedding width $128$.
We adopt this shallow network as a single configuration for all tasks rather than a task-specific setting.

\paragraph{Value learning.}
Following standard practice, ACSAC trains two Q-networks to improve stability.
ACSAC takes the minimum of the two Q-values on manipulation tasks and the mean on navigation tasks, for both the bootstrap target and policy extraction.
The discount factor $\gamma$ is $0.99$ on manipulation tasks and $0.995$ on navigation tasks.
ACSAC's bootstrap target uses the current online critic with the gradient stopped.
SEEM~\citep{SEEM} shows that LayerNorm bounds the critic's NTK, which allows stable training when the online and target networks share parameters.

\paragraph{Training and evaluation.}
\label{app:evaluation-protocol}
Following the offline-to-online protocol of~\citet{FQL, QC}, we train ACSAC with $1$M offline gradient steps and fine-tune it for $1$M online environment steps.
We evaluate the agent every $100$K steps using $50$ episodes, and report the offline and online success rates at $1$M and $2$M steps, respectively.

\paragraph{Results from prior works.}
For manipulation tasks, results of \texttt{IQL}, \texttt{ReBRAC}, \texttt{FQL}, \texttt{BFN}, \texttt{FQL-n}, \texttt{BFN-n}, \texttt{QC}, and \texttt{QC-FQL} are taken from~\citet{QC}\footnote{Released plot data: \url{https://github.com/ColinQiyangLi/qc/tree/main/plot_data}};
note that we use the same setting as this prior work, so the results are fully compatible.
For \texttt{DEAS} and the baselines on navigation tasks, we individually tune the method-specific hyperparameters for each domain and evaluate them under our protocol.
\section{Experimental Details}
\label{app:experiment-details}

\subsection{Environments, Tasks, and Datasets}
\label{app:domain}

We consider five manipulation domains and three navigation domains from OGBench~\citep{OGBench}, using their state-based single-task variants ($40$ tasks in total).
Each OGBench environment provides five evaluation goals, each of which defines a different task (\texttt{-singletask-task1} to \texttt{-singletask-task5})~\citep{FQL}.

Navigation tasks give a reward of $-1$ when the task is incomplete and $0$ when it is completed~\citep{FQL}.
For \texttt{scene-sparse} and \texttt{puzzle-3x3-sparse}, we sparsify the rewards to $\{-1, 0\}$, following~\citet{QC}.
For \texttt{cube-double/triple/quadruple}, we use the default semi-sparse reward: the negative of the number of remaining subtasks~\citep{FQL}.

We use the standard \texttt{navigate} datasets for navigation and \texttt{play} datasets for manipulation~\citep{OGBench}.
The \texttt{cube-triple} and \texttt{cube-quadruple} datasets contain $3$M and $100$M transitions, respectively, following~\citet{QC}.
\cref{tab:domain-metadata} lists the dataset size, episode length, and action dimension for each domain.

\begin{table}[ht]
    \renewcommand*{\ttdefault}{cmtt}
    \caption{\textbf{Domain metadata.}}
    \label{tab:domain-metadata}
    \begin{center}
    \normalsize
    \begin{tabular}{@{}cccc@{}}
        \toprule
        \textbf{Domain} & \textbf{Dataset Size} & \textbf{Episode Length} & \textbf{Action Dimension ($d$)} \\
        \midrule
        \texttt{scene-sparse-*} & $1$M & $750$ & $5$ \\
        \texttt{puzzle-3x3-sparse-*} & $1$M & $500$ & $5$ \\
        \texttt{cube-double-*} & $1$M & $500$ & $5$ \\
        \texttt{cube-triple-*} & $3$M & $1000$ & $5$ \\
        \texttt{cube-quadruple-100M-*} & $100$M & $1000$ & $5$ \\
        \texttt{antmaze-giant-*} & $1$M & $1000$ & $8$ \\
        \texttt{antsoccer-arena-*} & $1$M & $1000$ & $8$ \\
        \texttt{humanoidmaze-medium-*} & $2$M & $2000$ & $21$ \\
        \bottomrule
    \end{tabular}
\end{center}
\end{table}

\subsection{Methods and Hyperparameters}
\label{app:methods-hyperparameters}
\label{app:impl-hyperparams}
\label{app:baselines}

We use the official implementation of FQL~\citep{FQL} for IQL and ReBRAC, and that of QC~\citep{QC} for FQL, FQL-n, BFN, BFN-n, QC-FQL, and QC.
We reproduce DEAS with its official implementation~\citep{DEAS}, adapted to offline-to-online RL.
The hyperparameters for navigation tasks are based on~\citet{FQL,CGQ,DEAS}.

We report the common hyperparameters in \cref{tab:hyperparams-shared} and the domain-specific hyperparameters in \cref{tab:hyperparams-task}.
We use the same hyperparameter configuration for all tasks in each domain.

In \cref{tab:hyperparams-task}, $\alpha$ denotes the AWR inverse temperature for IQL and the BC coefficient for the other methods; $\tau$ is IQL's expectile coefficient.
For ReBRAC on navigation tasks, $\alpha$ is the actor BC coefficient, with the critic BC coefficient fixed at $0.01$.
$H$ denotes ACSAC's maximum chunk size or the baselines' fixed chunk size; $n$ is the multi-step return horizon, and $N$ is the rejection sampling size.

\label{app:hyperparams-shared}
\begin{table}[htbp]
    \caption{\textbf{Common hyperparameters.}}
    \label{tab:hyperparams-shared}
    \begin{center}
    \small
    \renewcommand{\arraystretch}{1.08}
    \begin{tabular}{@{}p{0.45\linewidth}p{0.47\linewidth}@{}}
        \toprule
        \textbf{Hyperparameter} & \textbf{Value} \\
        \midrule
        Optimizer & Adam \\
        Learning rate & $3\times10^{-4}$ \\
        Batch size & $256$ \\
        Discount factor $\gamma$ & $0.99$ (manipulation), $0.995$ (navigation) \\
        % Double Q aggregation & mean \\
        % Target network update rate $\tau$ & $0.005$ \\
        Critic ensemble size $K$ & $2$ \\
        UTD ratio & $1$ \\
        Number of flow steps $F$ & $10$ \\
        Number of offline training steps & $10^6$ \\
        Number of online environment steps & $10^6$ \\
        Network width & $512$ \\
        Network depth & $4$ hidden layers \\
        Nonlinearity & GELU \\
        \bottomrule
    \end{tabular}
\end{center}
\end{table}

\label{app:hyperparams-task}
\begin{table}[htbp]
    \renewcommand*{\ttdefault}{cmtt}
    \caption{\textbf{Domain-specific hyperparameters.}}
    \label{tab:hyperparams-task}
    \begin{center}
    \normalsize
    \setlength{\tabcolsep}{3pt}
    \renewcommand{\arraystretch}{1.22}
    \resizebox{\linewidth}{!}{%
    \begin{tabular}{@{}lcccccccccc@{}}
        \toprule
        \textbf{Domain} & \texttt{IQL} & \texttt{ReBRAC} & \texttt{FQL} & \texttt{BFN} & \texttt{FQL-n} & \texttt{BFN-n} & \texttt{QC-FQL} & \texttt{QC} & \texttt{DEAS} & \texttt{ACSAC} \\
        & $(\alpha,\tau)$ & $\alpha$ & $\alpha$ & $N$ & $(\alpha,n)$ & $(n,N)$ & $(\alpha,H)$ & $(H,N)$ & $(\alpha,H)$ & $(H,N)$ \\
        \midrule
        \texttt{scene-sparse} & $(10,0.9)$ & $0.1$ & $300$ & $4$ & $(100,5)$ & $(5,4)$ & $(300,5)$ & $(5,32)$ & $(3,4)$ & $(5,4)$ \\
        \texttt{puzzle-3x3-sparse} & $(10,0.9)$ & $0.1$ & $100$ & $4$ & $(100,5)$ & $(5,4)$ & $(300,5)$ & $(5,64)$ & $(3,8)$ & $(5,4)$ \\
        \texttt{cube-double} & $(0.3,0.9)$ & $0.1$ & $300$ & $4$ & $(100,5)$ & $(5,4)$ & $(300,5)$ & $(5,32)$ & $(300,4)$ & $(5,4)$ \\
        \texttt{cube-triple} & $(0.3,0.9)$ & $0.1$ & $300$ & $4$ & $(100,5)$ & $(5,4)$ & $(100,5)$ & $(5,32)$ & $(1,4)$ & $(5,4)$ \\
        \texttt{cube-quadruple} & $(0.3,0.9)$ & $0.1$ & $300$ & $4$ & $(100,5)$ & $(5,4)$ & $(100,5)$ & $(5,32)$ & $(1,4)$ & $(5,8)$ \\
        \texttt{antmaze-giant} & $(10,0.9)$ & $0.01$ & $10$ & $4$ & $(30,5)$ & $(5,4)$ & $(30,5)$ & $(5,32)$ & $(1,4)$ & $(5,4)$ \\
        \texttt{antsoccer-arena} & $(1,0.9)$ & $0.01$ & $30$ & $4$ & $(30,5)$ & $(5,4)$ & $(30,5)$ & $(5,32)$ & $(0.3,4)$ & $(5,4)$ \\
        \texttt{humanoidmaze-medium} & $(10,0.9)$ & $0.01$ & $100$ & $4$ & $(100,5)$ & $(5,4)$ & $(100,5)$ & $(5,32)$ & $(1,4)$ & $(3,4)$ \\
        \bottomrule
    \end{tabular}}
\end{center}
\end{table}

\clearpage
\section{Additional Results}
\label{app:additional-results}

\subsection{Computational Costs}
\label{app:computational-costs}

We run all experiments on NVIDIA RTX 3090 GPUs.
\cref{tab:param-count} compares parameter counts and per-step runtimes on \texttt{cube-triple}.
ACSAC has the smallest parameter count among the compared methods, using a single shallow Transformer critic to evaluate all prefix lengths.
Compared with \texttt{QC}, it uses $41\%$ fewer parameters and requires about $10\%$ less time per online step, while achieving a higher success rate ($98\%$ versus $64\%$).
Its offline runtime is comparable to \texttt{QC}, which shares the same rejection sampling backbone.

\paragraph{One-step flow distillation.}
ACSAC + Distill additionally distills the flow BC policy into a one-step policy~\citep{FQL,MAC}, and samples the $N$ candidate chunks from it for both the bootstrap targets in \cref{eq:h_target} and the joint argmax in \cref{eq:adaptive_policy}.
Each chunk then requires a single network call instead of $F{=}10$ flow steps.
This reduces the offline and online per-step runtimes by $29\%$ and $24\%$, respectively.
With online runtime comparable to \texttt{FQL}, ACSAC + Distill achieves $80\%$ success, outperforming both \texttt{FQL} ($18\%$) and \texttt{QC} ($64\%$).

\begin{table}[h]
    \renewcommand*{\ttdefault}{cmtt}
    \caption{\textbf{Training speed and parameter count on {\mdseries\texttt{cube-triple}}.}
    Offline measures one agent training step.
    Online measures one agent training step plus one environment step.
    Parameter counts are measured on \texttt{cube-triple} ($46$-dim observation, $5$-dim action).
    Online performance is the mean success rate (\%) across the five \texttt{cube-triple} tasks after $1$M online steps following $1$M offline updates.
    For all methods, performance is averaged over $4$ training seeds, with $95\%$ confidence intervals shown above each mean.}
    \label{tab:param-count}
    \begin{center}
    \small
    \setlength{\tabcolsep}{4pt}
    \renewcommand{\arraystretch}{1.55}
    \begin{tabular}{@{}lcccc@{}}
        \toprule
        \multirow{2}{*}{\textbf{Method}} & \multirow{2}{*}{\textbf{Online Performance (\%)}} & \multicolumn{2}{c}{\textbf{Training Time (ms/step)}} & \multirow{2}{*}{\shortstack{\textbf{Parameter}\\\textbf{Count}}} \\
        \cmidrule(lr){3-4}
        & & \textbf{Offline} & \textbf{Online} & \\
        \midrule
        \texttt{FQL}    & $\overset{\scriptscriptstyle[16,20]}{18}$ & $2.84$  & $10.26$ & $4\,911\,630$ \\
        \texttt{BFN}    & $\overset{\scriptscriptstyle[21,28]}{23}$ & $3.19$  & $11.82$ & $4\,094\,473$ \\
        \texttt{QC}     & $\overset{\scriptscriptstyle[61,66]}{64}$ & $11.50$ & $15.37$ & $4\,155\,933$ \\
        \texttt{QC-FQL} & $\overset{\scriptscriptstyle[46,61]}{53}$ & $4.06$  & $11.55$ & $4\,993\,590$ \\
        \texttt{DEAS}   & $\overset{\scriptscriptstyle[67,71]}{69}$ & $3.00$  & $5.42$  & $5\,231\,137$ \\
        \texttt{ACSAC}  & $\overset{\scriptscriptstyle[97,99]}{98}$ & $12.50$ & $13.76$ & $2\,455\,065$ \\
        \texttt{ACSAC + Distill} & $\overset{\scriptscriptstyle[78,83]}{80}$ & $8.92$ & $10.49$ & $3\,292\,722$ \\
        \bottomrule
    \end{tabular}
\end{center}
\end{table}

\subsection{Full Results}
\label{app:full-o2o-results}

\cref{tab:singletask} reports per-task offline-to-online RL results aligned with the domain-level summary in \cref{tab:multitask}.
The table covers five manipulation domains and three navigation domains.
\cref{fig:ogbench-domain-curves} provides summary plots by domain, followed by per-task curves in \cref{fig:full-ogbench-curves}.

\begin{table}[p]
\renewcommand*{\ttdefault}{cmtt}
\caption{\textbf{Complete OGBench offline-to-online RL results by task.} Each cell reports mean success rates (\%) over $4$ seeds at the end of $1$M offline training steps $\rightarrow$ $1$M additional online steps, with $95\%$ confidence intervals above each mean. Manipulation and navigation are averaged separately. The best nonzero offline and online results in each row are highlighted in color.}
\label{tab:singletask}
\begin{center}
\renewcommand{\arraystretch}{1.4}
\setlength{\extrarowheight}{2pt}
\setlength{\tabcolsep}{3pt}
\ifdefined\perftablebox\else\newsavebox{\perftablebox}\fi
\sbox{\perftablebox}{%
\begin{tabular}{cl>{\columncolor[HTML]{F5FAFE}}c>{\columncolor[HTML]{F5FAFE}}c>{\columncolor[HTML]{F5FAFE}}c>{\columncolor[HTML]{F5FAFE}}c>{\columncolor[HTML]{F7F6F9}}c>{\columncolor[HTML]{F7F6F9}}c>{\columncolor[HTML]{FDF7FA}}c>{\columncolor[HTML]{FDF7FA}}c>{\columncolor[HTML]{FDF7FA}}c>{\columncolor[HTML]{F2FBFA}}c}
\toprule
 &  & \texttt{IQL} & \texttt{ReBRAC} & \texttt{FQL} & \texttt{BFN} & \texttt{FQL-n} & \texttt{BFN-n} & \texttt{QC-FQL} & \texttt{QC} & \texttt{DEAS} & {\fontfamily{lmtt}\bfseries ACSAC} \\
\midrule
\multicolumn{12}{l}{\textbf{Manipulation}} \\
\bottomrule
\end{tabular}%
}
\ifdim\wd\perftablebox>\textwidth
  \sbox{\perftablebox}{\resizebox{\textwidth}{!}{\usebox{\perftablebox}}}%
\fi
\ifdim\dimexpr\ht\perftablebox+\dp\perftablebox\relax>.84\textheight
  \sbox{\perftablebox}{\resizebox*{!}{.84\textheight}{\usebox{\perftablebox}}}%
\fi
\makebox[\textwidth]{\usebox{\perftablebox}}
\end{center}
\end{table}

\begin{figure}[p]
    \centering
    \includegraphics[width=\textwidth,height=0.80\textheight,keepaspectratio]{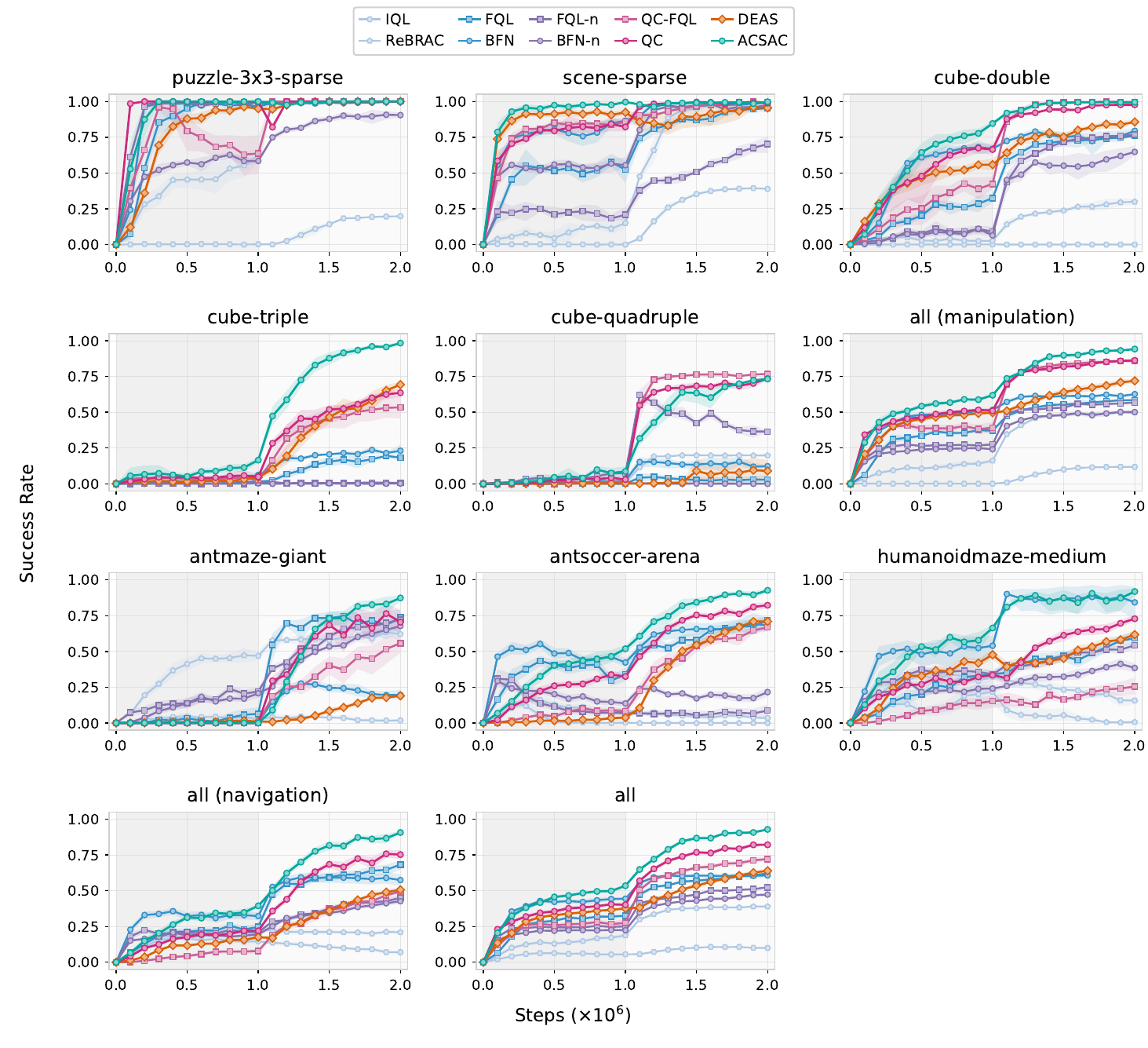}
    \caption{\textbf{OGBench offline-to-online RL results by domain.}
    Each domain panel aggregates five tasks, and the \texttt{all} panels average the $25$ manipulation tasks, the $15$ navigation tasks, and all $40$ tasks. The first $1$M steps correspond to offline training and the next $1$M steps correspond to online fine-tuning.
    Results are averaged over $4$ seeds with $95\%$ confidence intervals.}
    \label{fig:ogbench-domain-curves}
\end{figure}

\begin{figure}[p]
    \centering
    \includegraphics[width=\textwidth,height=0.80\textheight,keepaspectratio]{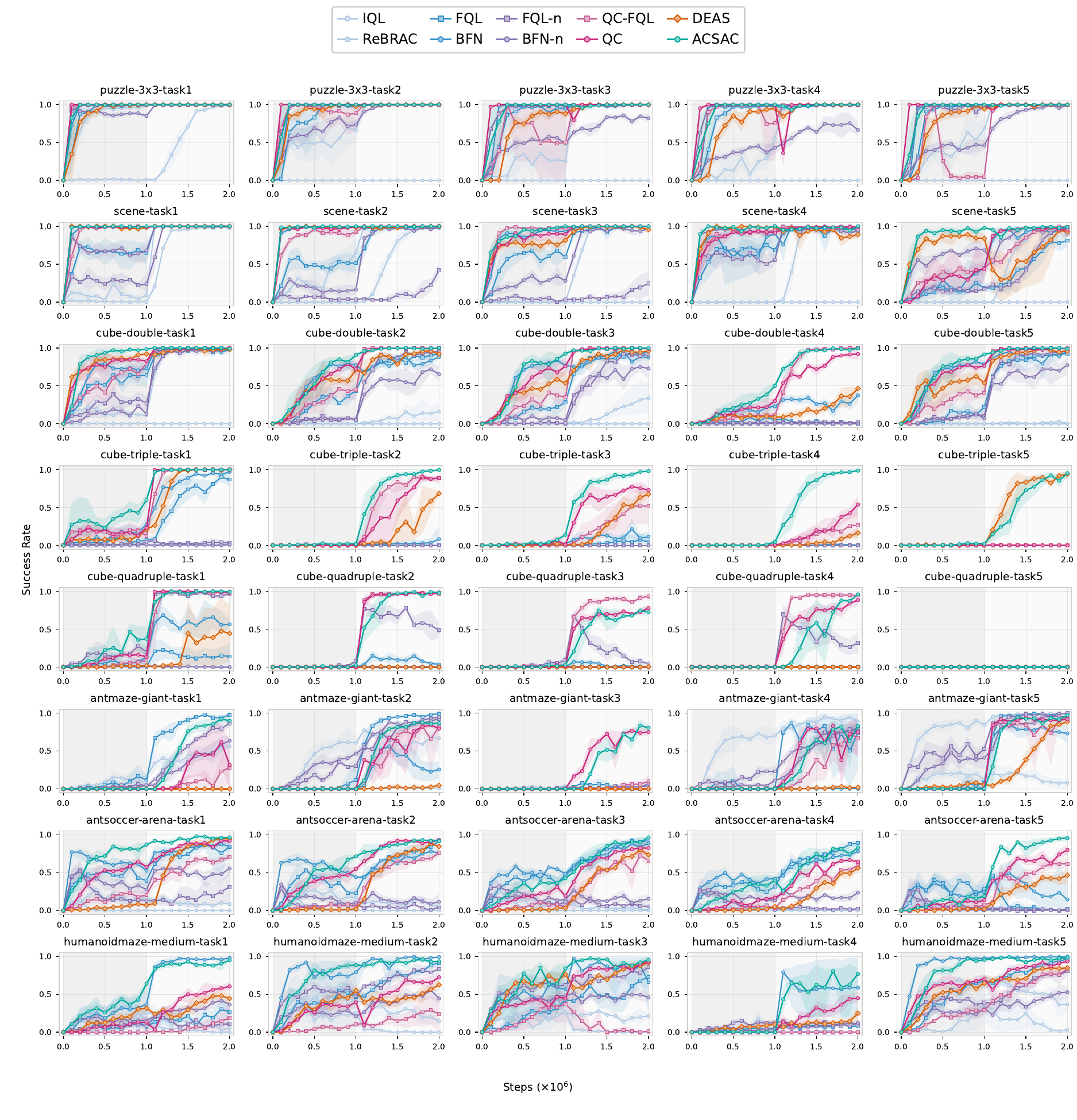}
    \caption{\textbf{Complete OGBench offline-to-online RL results by task.}
    The first $1$M steps correspond to offline training and the next $1$M steps correspond to online fine-tuning.
    Results are averaged over $4$ seeds with $95\%$ confidence intervals.}
    \label{fig:full-ogbench-curves}
\end{figure}

\clearpage

\subsection{Additional Evaluations with More Seeds}
\label{app:more-seeds}

\cref{tab:seed} presents additional evaluations of ACSAC, increasing the number of random seeds per task from $4$ to $8$.
We report mean success rates and $95\%$ confidence intervals at the end of offline training and online fine-tuning.

\begin{table}[htbp]
    \renewcommand*{\ttdefault}{cmtt}
    \caption{\textbf{OGBench offline-to-online RL results with $4$ and $8$ random seeds.}
    Entries report mean success rates (\%) at the end of offline training $\rightarrow$ online fine-tuning, with $95\%$ confidence intervals above each mean.
    Manipulation and navigation averages cover $25$ and $15$ tasks, respectively.}
    \label{tab:seed}
    \begin{center}
    \small
    \setlength{\tabcolsep}{6pt}
    \renewcommand{\arraystretch}{1.55}
    \newcommand{\seedresult}[6]{\ensuremath{\overset{\scriptscriptstyle[#2,#3]}{#1}\!\rightarrow\!\overset{\scriptscriptstyle[#5,#6]}{#4}}}
    \begin{tabular*}{\linewidth}{@{}l@{\extracolsep{\fill}}cc@{}}
        \toprule
        \textbf{Domain} & \texttt{ACSAC} (4-seed) & \texttt{ACSAC} (8-seed) \\
        \midrule
        \texttt{puzzle-3x3-sparse} & \seedresult{100}{100}{100}{100}{100}{100} & \seedresult{100}{100}{100}{100}{100}{100} \\
        \texttt{scene-sparse} & \seedresult{100}{99}{100}{100}{99}{100} & \seedresult{100}{99}{100}{100}{99}{100} \\
        \texttt{cube-double} & \seedresult{85}{84}{86}{100}{99}{100} & \seedresult{85}{84}{85}{100}{99}{100} \\
        \texttt{cube-triple} & \seedresult{17}{15}{19}{98}{97}{99} & \seedresult{18}{16}{19}{99}{98}{99} \\
        \texttt{cube-quadruple} & \seedresult{9}{5}{12}{73}{72}{75} & \seedresult{8}{5}{10}{73}{72}{74} \\
        \texttt{antmaze-giant} & \seedresult{0}{0}{0}{87}{86}{89} & \seedresult{0}{0}{0}{87}{86}{89} \\
        \texttt{antsoccer-arena} & \seedresult{52}{50}{54}{93}{92}{94} & \seedresult{52}{49}{54}{93}{92}{94} \\
        \texttt{humanoidmaze-medium} & \seedresult{66}{65}{68}{92}{90}{94} & \seedresult{65}{64}{67}{92}{90}{93} \\
        \midrule
        Manipulation Average & \seedresult{62}{61}{63}{94}{94}{95} & \seedresult{62}{61}{62}{94}{94}{95} \\
        Navigation Average & \seedresult{39}{38}{40}{91}{90}{92} & \seedresult{39}{38}{40}{91}{90}{91} \\
        \bottomrule
    \end{tabular*}
\end{center}
\end{table}

\subsection{Benefit of State-Dependent Chunk Size Selection}
\label{app:redeploy}

We examine state-dependent chunk size selection on \texttt{cube-triple} by changing only the chunk size rule after online fine-tuning, keeping the same flow BC policy and critic for each task and seed.
Fixed-$h$ uses a constant $h\in\{1,\ldots,5\}$, while uniform selection samples $h$ uniformly from this set at each replanning state.
Both select the highest-value length-$h$ prefix among $N=4$ candidate chunks of length $H=5$.
In \cref{tab:redeploy}, adaptive selection achieves $98\%$ success, compared with $32$--$69\%$ for fixed chunk sizes and $43\%$ for uniform selection, highlighting the benefit of selecting the chunk size according to the current state.

\begin{table}[h]
    \caption{\textbf{Effect of chunk size selection on \texttt{cube-triple}.}
    All variants use the same trained ACSAC checkpoints. Success rates (\%) are averaged over $5$ tasks, $4$ seeds, and $50$ episodes per task and seed.}
    \label{tab:redeploy}
    \begin{center}
    \begin{tabular*}{\linewidth}{@{}l@{\extracolsep{\fill}}ccccccc@{}}
        \toprule
        Chunk size & Adaptive & $h=1$ & $h=2$ & $h=3$ & $h=4$ & $h=5$ & Uniform \\
        \midrule
        Success (\%) & $\mathbf{98}$ & $32$ & $37$ & $54$ & $69$ & $63$ & $43$ \\
        \bottomrule
    \end{tabular*}
    \end{center}
\end{table}

\clearpage

\subsection{Adaptive Execution with a Larger Maximum Chunk Size}
\label{app:larger-chunk-size}

ACSAC also achieves strong performance with a larger maximum chunk size of $H=10$ on \texttt{cube-triple-task4}, a challenging multi-object manipulation task.
The online checkpoint used here achieves a $98\%$ success rate over $50$ evaluation episodes, after $1$M offline updates and $1$M online environment steps.

In \cref{fig:h_dist_app}, ACSAC selects long chunks ($h^\star=10$) for coarse approach and transport, and short chunks ($h^\star=1$--$2$) for fine alignment, grasping, and the final placement adjustment.
This pattern illustrates how ACSAC exploits a larger maximum horizon for sustained coarse motion while retaining frequent replanning for precise manipulation.

\begin{figure}[H]
    \centering
    \includegraphics[width=\linewidth]{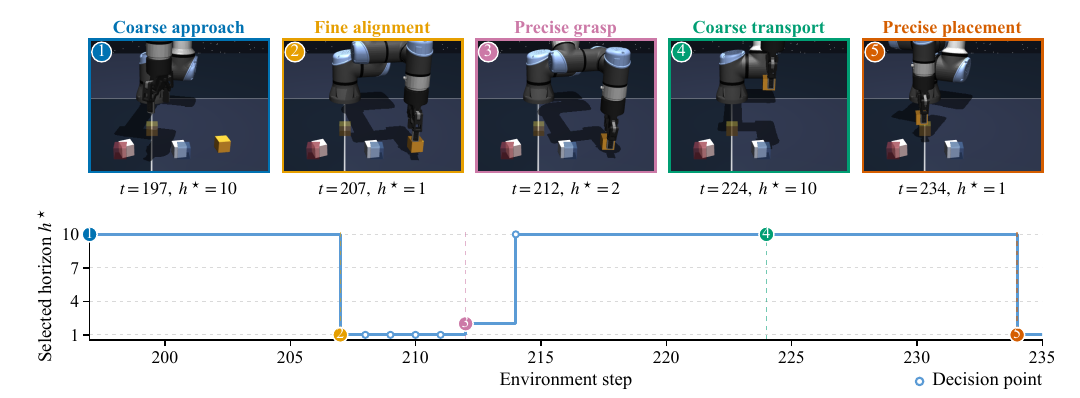}
    \vspace{-10pt}
    \caption{\textbf{Qualitative rollout of ACSAC with $H=10$.}
    Selected execution horizon $h^\star$ during a pick-and-place segment on \texttt{cube-triple-task4}.
    Numbered observations (top) correspond to the marked replanning decisions (bottom).}
    \label{fig:h_dist_app}
\end{figure}

\clearpage

\section{Theoretical Foundations}
\label{app:theory}

We analyze ACSAC as behavior-constrained variable-duration Bellman learning.
A length-$H$ chunk drawn from the flow BC policy is treated as a proposal path.
At each replanning state, ACSAC may execute any prefix of this path and replan afterwards.
We establish four results.
First, prefix-conditioned Q-values of different lengths are well-defined under prefix consistency (\cref{prop:causal_sufficient}) and share the unit of total discounted return under a common continuation policy (\cref{thm:cross_horizon}).
Second, the critic loss in \cref{eq:critic_loss} approximates a $\gamma$-contractive Bellman backup whose unique fixed point is the action-value function of the corresponding adaptive policy $\pi_\star$ (\cref{thm:contraction}).
Third, prefix truncation cannot increase the action-level mismatch between the flow BC policy and the behavior chunk distribution at the corresponding prefix length (\cref{thm:truncation_tv}).
Fourth, averaging per-horizon squared losses tightens a common gradient-variance bound under the stated covariance conditions, while preserving the sparse reward signal in each per-horizon target (\cref{lem:var_reduction}).

We work at the level of exact expectations, with deterministic transitions except in \cref{app:theory-extension}.
The Bellman analysis holds the flow BC policy, $N$, and $H$ fixed; it does not establish convergence of joint neural critic and flow BC policy training.
\cref{app:theory-extension} gives a semi-Markov interpretation of the executed prefixes and extends the Bellman results to stochastic transitions (\cref{prop:stochastic_transitions}).

\subsection{Setup and Notation}
\label{app:theory-setup}

\begin{assumption}[Deterministic MDP for analysis]
\label[assumption]{ass:det_mdp}
For the proofs we consider the deterministic version of the MDP in \cref{sec:preliminaries}, with $s_{t+1} = f(s_t, a_t)$ and bounded rewards $|r(s, a)| \le R_{\max}$, $\gamma \in (0, 1)$.
We use the sup-norm $\|Q\|_\infty := \sup_{s, a_{t:t+h}} |Q(s, a_{t:t+h})|$ on bounded $Q$-functions, and write the proofs for finite or discretized state-action spaces.
The same contraction arguments extend to bounded continuous spaces when the displayed maxima exist;
otherwise, maxima can be replaced by suprema.
The OGBench tasks evaluated in \cref{sec:experiments} are deterministic.
\end{assumption}

\paragraph{Variable-length action chunks.}
We use the body's chunk notation $a_{t:t+h} := (a_t, a_{t+1}, \ldots, a_{t+h-1}) \in \mathcal{A}^h$ throughout this appendix, with $h \in [H] := \{1, \ldots, H\}$ and $[N] := \{1, \ldots, N\}$.
We write
\begin{equation}
\label{eq:Aleq}
\Aleq \;:=\; \bigcup_{h=1}^{H} \mathcal{A}^h
\end{equation}
for the set of executable prefixes, with the convention that an element carries its own length $h$.
Under \cref{ass:det_mdp}, executing $a_{t:t+h}$ from $s_t$ produces the unique trajectory $s_{t+1} = f(s_t, a_t), \ldots, s_{t+h} = f(s_{t+h-1}, a_{t+h-1})$, with open-loop chunk return
\begin{equation}
\label{eq:chunk_return}
r_t^{(h)} \;:=\; \sum_{\tau=0}^{h-1}\gamma^\tau\, r(s_{t+\tau}, a_{t+\tau}),
\end{equation}
which is exactly the cumulative-reward term in the $h$-step target $G_h$ in \cref{eq:h_target}.

\paragraph{Behavior and proposal distributions.}
Let $\pi_\beta^H(\cdot \mid s_t) \in \Delta(\mathcal{A}^H)$ denote the true behavior chunk distribution under the data-collection policy of $\mathcal{D}$, and let $\pi_\theta^H(\cdot \mid s_t) \in \Delta(\mathcal{A}^H)$ denote the full-length law of the flow BC policy of \cref{sec:adaptive_extraction}.
For $h \in [H]$, $\pi_\theta^h(\cdot \mid s_t)$ denotes the distribution of the first $h$ actions when a full chunk is sampled from $\pi_\theta^H(\cdot \mid s_t)$,
\begin{equation}
\label{eq:prefix_marginal}
\pi_\theta^h(a_{t:t+h} \mid s_t) \;:=\; \int_{\mathcal{A}^{H-h}} \pi_\theta^H(a_{t:t+H} \mid s_t)\, d a_{t+h:t+H},
\end{equation}
and $\pi_\beta^h$ is defined analogously from $\pi_\beta^H$.
The integral becomes a sum in the discrete case.

\paragraph{Equal-value candidates.}
When several candidates attain the same maximum value, we choose one using a fixed state-independent rule, such as the first candidate in the stored order.
The contraction and fixed-point conclusions below do not depend on how ties are resolved;
the rule makes the policy extraction in \cref{def:induced_policy} below single-valued and is the rule assumed in \cref{rem:best_of_n_kl}.

\subsection{Prefix Values Are Well-Defined and Cross-Horizon Comparable}
\label{app:theory-comparability}

For ACSAC's joint argmax over $(n, h)$ to be meaningful, two conditions are needed.
The network output at the $h$-th position must represent a well-defined Q-value of the length-$h$ prefix, and the resulting Q-values across prefix lengths must lie on a single comparable scale.

\begin{definition}[Prefix consistency]
\label[definition]{def:prefix_consistency}
A network $\hat{Q}: \mathcal{S} \times \mathcal{A}^H \to \mathbb{R}^H$ is \emph{prefix-consistent} if its $h$-th output takes the same value on any two length-$H$ chunks that share the same first $h$ actions.
Equivalently, $\hat{Q}^{(h)}(s, a_{t:t+H})$ depends only on $(s, a_{t:t+h})$ and not on the suffix $a_{t+h:t+H}$.
\end{definition}

\begin{proposition}[Causal-masked Transformer is sufficient]
\label[proposition]{prop:causal_sufficient}
A decoder-only Transformer with causal self-attention and one output head per position satisfies prefix consistency.
\end{proposition}

\begin{proof}
Index the input tokens as $(s, a_t, \ldots, a_{t+H-1})$ with action $a_{t+h-1}$ at position $h$.
The causal mask restricts the hidden representation at position $h$ to attend only to positions $0, \ldots, h$, and induction over layers preserves this restriction.
The $h$-th output head reads only this hidden representation, so its value depends only on $(s, a_{t:t+h})$.
\end{proof}

\begin{remark}[Sufficient, not necessary]
\label[remark]{rem:sufficient}
A generic MLP that takes the entire chunk as input does not enforce prefix consistency by architecture.
It could in principle learn the suffix-invariance from data, but the property is not guaranteed.
ACSAC relies on the causal Transformer for this property by construction, not on data-driven invariance.
\end{remark}

\paragraph{Variable-horizon Bellman semantics.}
In \cref{thm:cross_horizon}, $V^\pi$ is the state-value function of the continuation policy $\pi$ starting at a replanning state.
The policy may itself select variable-length action chunks.
Prefix consistency ensures that each network output depends only on the prefix whose value it estimates.

\begin{proof}[Proof of \cref{thm:cross_horizon}]
Split the discounted return of executing $a_{t:t+h}$ and then following $\pi$ into the first $h$ rewards and the remaining return.
Under \cref{ass:det_mdp}, the prefix fixes $r_t^{(h)}$ and $s_{t+h}$, and the expected remaining return is $V^\pi(s_{t+h})$, discounted by $\gamma^h$.
This gives \cref{eq:Q_pi_prefix}; every $h$ thus measures the total discounted return from the same $s_t$ under the same continuation policy.

For a full chunk and any $h_1\le h_2\in[H]$, writing the additional reward over the segment $[h_1,h_2)$ as
\begin{equation}
\label{eq:partial_return}
r_{t+h_1}^{(h_2 - h_1)} \;:=\; \sum_{j=0}^{h_2 - h_1 - 1} \gamma^j r(s_{t+h_1 + j}, a_{t+h_1 + j}),
\end{equation}
we obtain the difference identity
\begin{equation}
\label{eq:diff_formula}
\begin{aligned}
&Q^\pi(s_t, a_{t:t+h_2}) - Q^\pi(s_t, a_{t:t+h_1}) \\
&\quad =\; \gamma^{h_1}\!\Bigl(r_{t+h_1}^{(h_2 - h_1)} + \gamma^{h_2 - h_1}\, V^\pi(s_{t+h_2}) - V^\pi(s_{t+h_1})\Bigr).
\end{aligned}
\end{equation}

Substitute $h=h_1,h_2$ into \cref{eq:Q_pi_prefix} and split $r_t^{(h_2)}-r_t^{(h_1)}=\gamma^{h_1}r_{t+h_1}^{(h_2-h_1)}$ to obtain \cref{eq:diff_formula}.
\end{proof}

\begin{remark}[Replanning-horizon interpretation]
\label[remark]{rem:comparability}
A longer prefix is preferred whenever the extra open-loop segment plus delayed continuation exceeds the value of replanning at $s_{t+h_1}$.
ACSAC's joint argmax over $(n, h)$ in \cref{eq:adaptive_policy} therefore compares replanning horizons on a common return scale.
\end{remark}

\subsection{Expected-Prefix-Max Bellman Backup}
\label{app:theory-bellman}

We now show that the per-horizon target inside the critic loss in \cref{eq:critic_loss} is an unbiased Monte Carlo sample of an expected-prefix-max Bellman backup $\Bnh$ that is a $\gamma$-contraction in sup-norm.
Its unique fixed point is the action-value function of the corresponding idealized adaptive-prefix policy $\pi_\star$.
The structure mirrors EMaQ~\citep{EMaQ}, which samples $N$ behavior actions and backs up their max;
ACSAC samples $N$ behavior chunks and backs up the max over all $NH$ executable prefixes.

\begin{definition}[ACSAC expected-prefix-max Bellman backup]
\label[definition]{def:bnh}
For bounded $Q : \mathcal{S} \times \Aleq \to \mathbb{R}$, the \emph{ACSAC Bellman backup} is
\begin{equation}
\label{eq:bnh}
\begin{aligned}
&(\Bnh Q)(s_t, a_{t:t+h}) \;:=\; r_t^{(h)} \\
&\quad +\; \gamma^h\, \mathbb{E}_{\tilde a^{(1:N)}_{t+h:t+h+H} \sim \pi_\theta^H(\cdot \mid s_{t+h})}\!\left[\,\max_{n \in [N],\, k \in [H]} Q\!\left(s_{t+h},\, \tilde a^{(n)}_{t+h:t+h+k}\right)\right].
\end{aligned}
\end{equation}
The tilde distinguishes the $N$ i.i.d.\ proposal chunks at the bootstrap state $s_{t+h}$ from the chunk $a_{t:t+H}$ on the regression side, and the maximum ranges over the $NH$ candidate prefixes formed from these proposals.
\end{definition}

\begin{lemma}[Max non-expansiveness over a common index set]
\label[lemma]{lem:max_nonexpansive}
For any finite index set $\mathcal{I}$ and any real-valued $g_1, g_2: \mathcal{I} \to \mathbb{R}$,
\begin{equation}
\label{eq:max_nonexpansive}
\bigl|\max_{i \in \mathcal{I}} g_1(i) - \max_{i \in \mathcal{I}} g_2(i)\bigr| \;\le\; \max_{i \in \mathcal{I}} |g_1(i) - g_2(i)|.
\end{equation}
\end{lemma}

\begin{proof}
Let $i_1 \in \arg\max_i g_1(i)$ and $i_2 \in \arg\max_i g_2(i)$.
Then $\max_i g_1(i) - \max_i g_2(i) \le g_1(i_1) - g_2(i_1) \le |g_1(i_1) - g_2(i_1)| \le \max_i |g_1(i) - g_2(i)|$, and the symmetric direction yields the absolute value.
This is a deterministic, set-theoretic inequality, so it holds even when the indexed values $g_1(i), g_2(i)$ are correlated across $i$, which is what allows EMaQ's contraction argument to transport to ACSAC's $NH$ correlated prefix candidates~\citep{EMaQ}.
\end{proof}

\begin{proof}[Proof of the $\gamma$-contraction in \cref{thm:contraction}]
Fix $(s_t, a_{t:t+h})$ and let $s' = s_{t+h}$.
The shared $r_t^{(h)}$ cancels, so
$|(\Bnh Q_1)(s_t, a_{t:t+h}) - (\Bnh Q_2)(s_t, a_{t:t+h})| = \gamma^h \bigl| \mathbb{E}[\max_{n,k} Q_1(s', \cdot)] - \mathbb{E}[\max_{n,k} Q_2(s', \cdot)] \bigr|$.
By Jensen's inequality and \cref{lem:max_nonexpansive} applied pointwise to each realized proposal sample,
$\bigl| \mathbb{E}[\max Q_1] - \mathbb{E}[\max Q_2] \bigr| \le \mathbb{E}\bigl[\max_{n,k} |Q_1(s', \cdot) - Q_2(s', \cdot)|\bigr] \le \|Q_1 - Q_2\|_\infty$.
Hence $|(\Bnh Q_1) - (\Bnh Q_2)| \le \gamma^h \|Q_1 - Q_2\|_\infty \le \gamma\, \|Q_1 - Q_2\|_\infty$.
By the Banach fixed-point theorem, $\Bnh$ has a unique fixed point $\Qnh$, and repeated application of $\Bnh$ converges to $\Qnh$ from any bounded initial Q-function.
\end{proof}

\begin{definition}[Critic-induced extraction policy]
\label[definition]{def:induced_policy}
For any bounded $Q$, the extraction policy $\pi_\star^Q$ acts at state $s_t$ as follows.
Draw $a^{(1)}_{t:t+H}, \ldots, a^{(N)}_{t:t+H} \stackrel{\mathrm{i.i.d.}}{\sim} \pi_\theta^H(\cdot \mid s_t)$, set
\begin{equation}
\label{eq:argmax_pair}
(n^\star, h^\star) \;=\; \arg\max_{n \in [N],\, h \in [H]} Q\!\left(s_t,\, a^{(n)}_{t:t+h}\right),
\end{equation}
and execute the prefix $a^{(n^\star)}_{t:t+h^\star}$, treating it as one temporally extended action.
We write $\pi_\star := \pi_\star^{\Qnh}$ for the idealized adaptive-prefix policy induced by the operator fixed point;
the deployed implementation in \cref{sec:adaptive_extraction} applies the same extraction rule to the learned critic $Q_\phi$.
\end{definition}

\begin{proof}[Proof of the fixed-point identity in \cref{thm:contraction}]
\cref{eq:argmax_pair} gives, for any draw of the $N$ proposals,
$\Qnh(s_t, a^{(n^\star)}_{t:t+h^\star}) = \max_{n \in [N],\, k \in [H]} \Qnh(s_t, a^{(n)}_{t:t+k})$.
Taking the expectation over the proposal draws, the inner expected max in \cref{eq:bnh} equals $\mathbb{E}_{\pi_\star}[\Qnh(s_t, \cdot)]$, so the fixed-point equation $\Qnh = \Bnh \Qnh$ reads
\begin{equation*}
\Qnh(s_t, a_{t:t+h}) \;=\; r_t^{(h)} + \gamma^h\, \mathbb{E}_{a' \sim \pi_\star(\cdot \mid s_{t+h})}\bigl[\Qnh(s_{t+h}, a')\bigr].
\end{equation*}
This is the variable-horizon Bellman equation for $\pi_\star$ and uniquely identifies its action-value function by the same contraction argument as \cref{thm:contraction}, with the fixed policy $\pi_\star$ in place of the joint max.
We may therefore write the fixed point as $Q^{\pi_\star}$ in the rest of the appendix.
Together with \cref{thm:cross_horizon}, this shows that ACSAC's fixed-point prefix values share the same adaptive continuation $V^{\pi_\star}$.
\end{proof}

\begin{proposition}[Critic targets are Monte Carlo Bellman samples]
\label[proposition]{prop:loss_correspondence}
Fix any bounded $Q : \mathcal{S} \times \Aleq \to \mathbb{R}$ and define the per-horizon target
\begin{equation}
\label{eq:G_h_sample}
\hat G_h(Q) \;:=\; r_t^{(h)} \;+\; \gamma^h \max_{n \in [N],\, k \in [H]} Q\!\left(s_{t+h},\, \tilde a^{(n)}_{t+h:t+h+k}\right),
\end{equation}
with $\tilde a^{(1)}_{t+h:t+h+H}, \ldots, \tilde a^{(N)}_{t+h:t+h+H} \stackrel{\mathrm{i.i.d.}}{\sim} \pi_\theta^H(\cdot \mid s_{t+h})$.
$\hat G_h(Q)$ is the explicit-max sample form of the body's bootstrap target $G_h$ in \cref{eq:h_target}, with $\pi_\star$ replaced by the joint $\arg\max$ and the bootstrap critic treated as a free argument; the body's $G_h(s_t, a_{t:t+h})$ corresponds to $\hat G_h(Q_{\bar\phi})$ at this conditioning point.
Under \cref{ass:det_mdp},
\begin{equation}
\label{eq:G_h_expectation}
\mathbb{E}\!\left[\hat G_h(Q) \,\middle|\, s_t, a_{t:t+h}\right] \;=\; (\Bnh Q)(s_t, a_{t:t+h}),
\end{equation}
and the conditional squared-error loss decomposes as
\begin{equation}
\label{eq:bias_var}
\begin{aligned}
&\mathbb{E}\!\left[(Q_\phi(s_t, a_{t:t+h}) - \hat G_h(Q))^2 \,\middle|\, s_t, a_{t:t+h}\right] \\
&\quad =\; \bigl(Q_\phi(s_t, a_{t:t+h}) - (\Bnh Q)(s_t, a_{t:t+h})\bigr)^2 + \mathrm{Var}\!\left[\hat G_h(Q) \,\middle|\, s_t, a_{t:t+h}\right].
\end{aligned}
\end{equation}
Hence the empirical critic loss in \cref{eq:critic_loss} is a Monte Carlo regression toward the Bellman backup $\Bnh Q$, not an exact squared Bellman residual for any single sampled target.
In the tabular exact-expectation regime, fitted iteration with $\Bnh$ converges to $\Qnh$.
\end{proposition}

\begin{proof}
\cref{eq:G_h_expectation} follows from \cref{def:bnh}, since under \cref{ass:det_mdp} the rewards $\{r_{t+\tau}\}_{\tau<h}$ summed in $r_t^{(h)}$ are deterministic given $(s_t, a_{t:t+h})$, and the only randomness in $\hat G_h(Q)$ comes from the proposal sampling that defines $\Bnh Q$.
The decomposition in \cref{eq:bias_var} is the standard bias-variance identity for the squared error of a deterministic prediction $Q_\phi(s_t, a_{t:t+h})$ against a stochastic target $\hat G_h(Q)$ with mean $(\Bnh Q)(s_t, a_{t:t+h})$.
\end{proof}

\begin{remark}[Bounded fixed-point values]
\label[remark]{rem:bounded_values}
With rewards in $[-R_{\max}, R_{\max}]$, the contraction in \cref{thm:contraction} maps the interval $[-R_{\max}/(1-\gamma),\, R_{\max}/(1-\gamma)]$ to itself pointwise, so $\Qnh$ is bounded by the reward scale uniformly in $(s_t, a_{t:t+h})$.
This is an operator-level statement;
the iterates of a learned neural critic may transiently exceed this interval without explicit clipping.
\end{remark}

\subsection{Prefix Truncation Does Not Increase Behavior Mismatch}
\label{app:theory-truncation}

ACSAC's bootstrap target queries the critic at an adaptively-chosen prefix of a behavior-generated chunk.
A natural concern is whether shorter prefixes drift further from the behavior support than the full chunk does.
We rule this out at the level of total variation distance, using only the elementary fact that marginalization is non-expansive in TV.

\begin{lemma}[Marginalization is TV non-expansive]
\label[lemma]{lem:tv_marginal}
For any two distributions $\mu, \nu$ on $\mathcal{A}^H$ and any $h \in [H]$, let $\mu_h, \nu_h$ denote their length-$h$ prefix marginals as in \cref{eq:prefix_marginal}.
Then
\begin{equation}
\label{eq:tv_marginal}
D_{\mathrm{TV}}(\mu_h, \nu_h) \;\le\; D_{\mathrm{TV}}(\mu, \nu).
\end{equation}
\end{lemma}

\begin{proof}
Using the convention $D_{\mathrm{TV}}(p, q) = \tfrac{1}{2}\sum_x |p(x) - q(x)|$ and the triangle inequality,
$\sum_{a_{t:t+h}} \bigl|\sum_{a_{t+h:t+H}} (\mu - \nu)(a_{t:t+H})\bigr| \le \sum_{a_{t:t+H}} |\mu(a_{t:t+H}) - \nu(a_{t:t+H})|$.
Dividing by $2$ gives \cref{eq:tv_marginal}.
The continuous case is analogous, with sums replaced by integrals.
\end{proof}

\begin{proposition}[Prefix truncation and behavior mismatch]
\label[proposition]{thm:truncation_tv}
Let $\pi_\theta^H$ and $\pi_\beta^H$ be the flow BC and behavior chunk distributions, and let $\pi_\theta^h$ and $\pi_\beta^h$ be their length-$h$ prefix marginals.
For every $h\in[H]$,
\begin{equation}
\label{eq:tv_prefix_bound}
D_{\mathrm{TV}}\bigl(\pi_\theta^h(\cdot\mid s_t),\pi_\beta^h(\cdot\mid s_t)\bigr)
\le D_{\mathrm{TV}}\bigl(\pi_\theta^H(\cdot\mid s_t),\pi_\beta^H(\cdot\mid s_t)\bigr).
\end{equation}
\end{proposition}

\begin{proof}[Proof of \cref{thm:truncation_tv}]
Apply \cref{lem:tv_marginal} with $\mu = \pi_\theta^H(\cdot \mid s_t)$ and $\nu = \pi_\beta^H(\cdot \mid s_t)$.
\end{proof}

\begin{remark}[Implication for the TD target]
\label[remark]{rem:trunc_td_target}
\cref{thm:truncation_tv} concerns the proposal distribution before value-based selection.
Shortening a full behavior-generated chunk to any prefix does not increase its action-level mismatch from the corresponding behavior prefix marginal.
The set of prefixes used by the bootstrap target in \cref{eq:critic_loss} is therefore obtained from behavior-generated chunks without any length-dependent shift caused by truncation itself.
\end{remark}

\begin{remark}[Selection adds a separate distribution shift]
\label[remark]{rem:best_of_n_kl}
Value-based rejection sampling does change the selected action distribution.
QC analyzes the behavior constraint induced by rejection sampling from a fixed proposal distribution~\citep{QC}.
Fix $s_t$ and a bounded $Q$.
For a full-length chunk $a = a_{t:t+H}$, let $h_Q(a) := \min \arg\max_{h \in [H]} Q(s_t, a_{t:t+h})$ and $T_Q(a) := a_{t:t+h_Q(a)} \in \Aleq$, so that $T_Q$ keeps the highest-value prefix of $a$ and resolves equal values toward the shorter prefix.
Since $\max_{n \in [N],\, h \in [H]} Q(s_t, a^{(n)}_{t:t+h}) = \max_{n \in [N]} Q(s_t, T_Q(a^{(n)}))$, the joint argmax in \cref{eq:argmax_pair} under the tie rule for equal-value candidates in \cref{app:theory-setup} is equivalent to a two-step selection.
It first picks the candidate with the highest score $Q(s_t, T_Q(a^{(n)}))$, equal scores going to the smaller index, and then applies $T_Q$ to the selected chunk.
Let $\pi^H_{N\text{-best}}(\cdot \mid s_t)$ denote the distribution of the chunk selected in this way, and write $(T_Q)_\# \mu$ for the distribution of $T_Q(a)$ when $a \sim \mu$.
The extraction policy of \cref{def:induced_policy} then satisfies $\pi_\star^Q = (T_Q)_\# \pi^H_{N\text{-best}}$, and the data-processing inequality together with QC's closed-form bound~\citep{QC} gives
\begin{equation}
\label{eq:selection_kl}
D_{\mathrm{KL}}\bigl(\pi_\star^Q(\cdot \mid s_t) \,\big\|\, (T_Q)_\# \pi_\theta^H(\cdot \mid s_t)\bigr)
\;\le\;
D_{\mathrm{KL}}\bigl(\pi^H_{N\text{-best}}(\cdot \mid s_t) \,\big\|\, \pi_\theta^H(\cdot \mid s_t)\bigr)
\;\le\;
\log N - \frac{N-1}{N}.
\end{equation}
Ties are broken by candidate index and not by the actions themselves, so the selected chunk, given its score, is distributed like a single proposal chunk with that score.
The last inequality in \cref{eq:selection_kl} therefore also holds when scores can coincide.
The upper bound in \cref{eq:selection_kl} depends only on the number $N$ of sampled chunks, not on the $NH$ scored prefixes, and is uniform over $H$.
The reference distribution $(T_Q)_\# \pi_\theta^H$ in \cref{eq:selection_kl} is the full-length law of the flow BC policy truncated by the same rule.
The bound therefore compares $\pi_\star^Q$ with this truncated proposal law, not with the fixed-length prefix marginal $\pi_\theta^{h}$ of \cref{eq:prefix_marginal} at the selected length.
By contrast, \cref{thm:truncation_tv} concerns prefix truncation before value-based selection.
The mismatch between replay trajectories and open-loop execution is a separate issue, discussed briefly at the end of \cref{app:theory-extension}.
\end{remark}

\subsection{Per-Horizon Loss Averaging Stabilizes Updates}
\label{app:theory-multi-horizon}

The multi-step targets $G_h$ in \cref{eq:critic_loss} accumulate rewards over $h$ steps, and the variance of this cumulative sum tends to grow with $h$.
ACSAC averages all $H$ per-horizon squared losses at the gradient level, in line with T-SAC's gradient-level averaging~\citep{T-SAC} and TOP-ERL's multi-horizon Transformer supervision~\citep{TOP-ERL}.
A standard variance-of-mean argument explains this design choice.

\paragraph{Cumulative-reward variance.}
For the cumulative-reward part of $G_h$, expansion gives
\begin{equation}
\label{eq:reward_sum_var}
\mathrm{Var}\!\left[r_t^{(h)}\right] \;=\; \sum_{i=0}^{h-1}\sum_{j=0}^{h-1} \gamma^{i+j}\, \mathrm{Cov}(r_{t+i},\, r_{t+j}),
\end{equation}
which typically increases with $h$, especially when $\gamma$ is close to $1$ and per-step rewards are positively correlated.

\paragraph{Per-horizon gradient.}
Let $\delta_h(\phi) := Q_\phi(s_t, a_{t:t+h}) - G_h$ denote the per-horizon Bellman residual, and let $g_h := 2\delta_h(\phi)\, \nabla_\phi Q_\phi(s_t, a_{t:t+h})$ denote its per-horizon loss gradient.
The averaged loss $\frac{1}{H}\sum_h \delta_h^2$ in \cref{eq:critic_loss} produces the gradient $\bar g := \frac{1}{H}\sum_{h=1}^{H} g_h$, and let $\tilde g_h := g_h - \mathbb{E}[g_h]$ denote its centered version.

\begin{lemma}[Variance reduction by averaging]
\label[lemma]{lem:var_reduction}
Let $H\ge2$ and $\sigma^2>0$. Suppose the centered per-horizon gradients satisfy $\mathbb{E}[\|\tilde g_h\|^2] \le \sigma^2$ for every $h$ and $\mathbb{E}[\langle \tilde g_h, \tilde g_{h'}\rangle] \le \rho\,\sigma^2$ for some $\rho \in [-1/(H-1), 1]$ and every $h \ne h'$.
Then the centered average $\bar{\tilde g} := \frac{1}{H}\sum_h \tilde g_h$ satisfies
\begin{equation}
\label{eq:var_reduction}
\mathrm{Var}(\bar g) \;:=\; \mathbb{E}\!\left[\big\|\bar{\tilde g}\big\|^2\right] \;\le\; \sigma^2\!\left[\rho + \frac{1-\rho}{H}\right].
\end{equation}
Whenever $\rho < 1$, this is strictly less than $\sigma^2$, the common upper bound on the per-horizon gradient variances.
\end{lemma}

\begin{proof}
Expand $\|\bar{\tilde g}\|^2 = H^{-2}\sum_{h, h'} \langle \tilde g_h, \tilde g_{h'}\rangle$ and take expectations.
The diagonal terms contribute at most $H^{-2} \cdot H \sigma^2 = \sigma^2/H$, and the off-diagonal terms contribute at most $H^{-2} \cdot H(H-1)\rho\sigma^2$, which sum to the right-hand side of \cref{eq:var_reduction}.
\end{proof}

\begin{remark}[Imperfect correlation in ACSAC]
\label[remark]{rem:correlation}
This is the setting targeted by T-SAC's gradient-level averaging argument~\citep{T-SAC}.
The per-horizon gradients $g_h$ in ACSAC share the same critic parameters $\phi$ and may be correlated.
Their residuals $\delta_h$ use different cumulative-reward sums and different bootstrap states $s_{t+h}$, so their gradients need not be perfectly aligned.
When the covariance condition holds with $\rho<1$, \cref{lem:var_reduction} gives a tighter common variance bound; it does not imply lower variance than every individual single-horizon gradient.
Target-level averaging, in contrast, would collapse the distinct per-horizon signals.
\end{remark}

\begin{remark}[Sparse reward signals are preserved]
\label[remark]{rem:sparse_preserved}
ACSAC averages per-horizon \emph{losses}, not the targets themselves.
A non-zero reward $r_{t+\tau}$ enters every $G_h$ with $h > \tau$ as a discounted summand, so the sparse signal propagates into all $H - \tau$ residuals simultaneously.
This contrasts with target-level averaging, which would replace the per-horizon targets by a single averaged scalar and could dilute sparse rewards~\citep{T-SAC}.
\end{remark}

\subsection{Semi-Markov Interpretation and Stochastic Extension}
\label{app:theory-extension}

\paragraph{Semi-Markov interpretation.}
At each replanning state $s_{t_m}$, ACSAC selects a prefix of length $h_m$ and commits to its actions until the next replanning state at $t_{m+1} = t_m + h_m$.
The selected prefix is one temporally extended action of duration $h_m$, with open-loop chunk return $r_{t_m}^{(h_m)}$ and continuation discount $\gamma^{h_m}$.
The decisions at replanning states thus form a semi-Markov decision process (SMDP) over $\mathcal{S}$ whose actions are the executable prefixes~\citep{SMDP}.
DEAS likewise treats a fixed-length action sequence as a single decision unit~\citep{DEAS}.
Regrouping the per-step rewards at replanning states leaves the discounted return unchanged, since $\gamma^{t_m}\gamma^{j} = \gamma^{t_m + j}$.
Hence $V^\pi$ and $Q^\pi$ in \cref{thm:cross_horizon} measure the original per-step discounted return, evaluated from replanning states.

\paragraph{Stochastic transitions.}
We now drop the deterministic dynamics in \cref{ass:det_mdp} and keep its other conventions, so that $s_{t+\tau+1} \sim T(\cdot \mid s_{t+\tau}, a_{t+\tau})$ for $\tau = 0, \ldots, h-1$ with $T$ the transition dynamics of \cref{sec:preliminaries}.
Executing $a_{t:t+h}$ from $s_t$ now produces a random state sequence $s_{t+1:t+h+1} = (s_{t+1}, \ldots, s_{t+h})$, so the chunk return $r_t^{(h)}$ in \cref{eq:chunk_return} and the bootstrap state $s_{t+h}$ depend on the sampled transitions.
The backup of \cref{def:bnh} takes an outer expectation over this sequence,
\begin{equation}
\label{eq:bnh_stochastic}
\begin{aligned}
&(\Bnh Q)(s_t, a_{t:t+h}) \;:=\; \mathbb{E}_{s_{t+1:t+h+1} \mid s_t,\, a_{t:t+h}}\!\Bigl[\, r_t^{(h)} \\
&\quad +\; \gamma^h\, \mathbb{E}_{\tilde a^{(1:N)}_{t+h:t+h+H} \sim \pi_\theta^H(\cdot \mid s_{t+h})}\!\left[\,\max_{n \in [N],\, k \in [H]} Q\!\left(s_{t+h},\, \tilde a^{(n)}_{t+h:t+h+k}\right)\right]\Bigr].
\end{aligned}
\end{equation}
The outer expectation fixes $s_t$ and the whole prefix $a_{t:t+h}$ first and then samples the $h$ transitions, and the inner expectation samples the $N$ proposal chunks at the realized $s_{t+h}$ exactly as in \cref{def:bnh}.
Under \cref{ass:det_mdp} the outer expectation is trivial and \cref{eq:bnh_stochastic} reduces to \cref{eq:bnh}, so we keep the symbols $\Bnh$ and $\Qnh$.

\begin{proposition}[Stochastic transitions]
\label[proposition]{prop:stochastic_transitions}
Let transitions be stochastic as above, and fix the flow BC policy $\pi_\theta$, rejection sampling size $N$, and maximum chunk size $H$.
(i) Let $\pi$ be any continuation policy whose distribution over executed prefixes at a replanning state depends only on that state, such as the extraction policy $\pi_\star^Q$ of \cref{def:induced_policy} for any bounded $Q$.
Then \cref{eq:Q_pi_prefix} holds with its right-hand side replaced by the expectation over the executed transitions,
\begin{equation}
\label{eq:Q_pi_prefix_stochastic}
Q^\pi(s_t, a_{t:t+h}) \;=\; \mathbb{E}_{s_{t+1:t+h+1} \mid s_t,\, a_{t:t+h}}\!\left[\, r_t^{(h)} + \gamma^h\, V^\pi(s_{t+h}) \right].
\end{equation}
(ii) The backup $\Bnh$ in \cref{eq:bnh_stochastic} is a $\gamma$-contraction in sup-norm on bounded $Q : \mathcal{S} \times \Aleq \to \mathbb{R}$.
Its unique fixed point $\Qnh$ equals $Q^{\pi_\star}$, where $\pi_\star = \pi_\star^{\Qnh}$ is the extraction policy of \cref{def:induced_policy} induced by this fixed point.
\end{proposition}

\begin{proof}
For (i), split the return of executing $a_{t:t+h}$ and then following $\pi$ into the first $h$ rewards and the remaining return.
Conditional on the realized state sequence $s_{t+1:t+h+1}$, the first $h$ rewards are fixed.
The transitions are Markov and $\pi$ chooses at $s_{t+h}$ from that state alone, so the expected remaining return is $V^\pi(s_{t+h})$.
Taking the expectation over the executed transitions gives \cref{eq:Q_pi_prefix_stochastic}.

For the contraction in (ii), fix $(s_t, a_{t:t+h})$.
The reward term cancels in $\Bnh Q_1 - \Bnh Q_2$.
Jensen's inequality moves the absolute value inside both expectations, and \cref{lem:max_nonexpansive} applied to each realized $(s_{t+h}, \tilde a^{(1:N)}_{t+h:t+h+H})$ bounds the integrand by $\|Q_1 - Q_2\|_\infty$.
Hence $|(\Bnh Q_1 - \Bnh Q_2)(s_t, a_{t:t+h})| \le \gamma^h \|Q_1 - Q_2\|_\infty \le \gamma\, \|Q_1 - Q_2\|_\infty$, and the Banach fixed-point theorem applies as in \cref{thm:contraction}.

For the fixed-point identity in (ii), the proposals in \cref{eq:bnh_stochastic} are drawn at the realized $s_{t+h}$ from the same law as the fresh proposals of $\pi_\star$.
As in the proof of \cref{thm:contraction}, the inner expected maximum of $\Qnh$ therefore equals $\mathbb{E}_{a' \sim \pi_\star(\cdot \mid s_{t+h})}[\Qnh(s_{t+h}, a')]$ at every realized $s_{t+h}$.
Substituting this into $\Qnh = \Bnh \Qnh$ shows that $\Qnh$ satisfies the Bellman evaluation equation of the fixed policy $\pi_\star$.
This equation has a unique bounded solution, since its right-hand side is a $\gamma$-contraction by the argument above with $\pi_\star$ held fixed in place of the joint maximum.
By (i) and $V^{\pi_\star}(s) = \mathbb{E}_{a' \sim \pi_\star(\cdot \mid s)}[Q^{\pi_\star}(s, a')]$, the function $Q^{\pi_\star}$ satisfies the same equation.
Uniqueness gives $\Qnh = Q^{\pi_\star}$.
\end{proof}

With $H = 1$, \cref{eq:bnh_stochastic} is exactly EMaQ's expected-max Q operator over single actions, including its expectation over the next state~\citep{EMaQ}.
\cref{prop:stochastic_transitions} concerns the backup operator, not its estimation from data.
Whether the sampled targets $\hat G_h$ of \cref{prop:loss_correspondence} estimate this backup under stochastic transitions depends on how the data were collected.
A logged trajectory, conditioned on its recorded actions, need not have the same law as executing those actions open loop from the same state.
DQC bounds the resulting gap between the nominal value estimated from such data and the actual value of open-loop execution in terms of the open-loop consistency of the dataset~\citep{DQC}.

\end{document}